%% file: ms.tex
\documentclass{article}

\usepackage[nonatbib,final]{neurips_2020}
     
\usepackage[numbers]{natbib}

\usepackage[utf8]{inputenc} 
\usepackage[T1]{fontenc}    
\usepackage{url}            
\usepackage{booktabs}       
\usepackage{amsfonts}       
\usepackage{nicefrac}       
\usepackage{microtype}      
\usepackage{graphicx}
\usepackage{subfigure}
\usepackage{enumitem}

\input{preamble/preamble.tex}

\input{preamble/preamble_math.tex}

\input{preamble/preamble_acronyms.tex}
\definecolor{blue}{rgb}{0.0, 0.0, 0.5}
\hypersetup{colorlinks,linkcolor={blue},citecolor={blue},urlcolor={blue}}

\title{General Control Functions for Causal Effect Estimation from Instrumental Variables}

\usepackage{footmisc}

\author{
  Aahlad Puli\\
  Computer Science\\
  New York University\\
  \texttt{aahlad@nyu.edu} \\
  \And
  Rajesh Ranganath \\
  Computer Science, Center for Data Science \\
  New York University\\
  \texttt{rajeshr@cims.nyu.edu} \\
}

\begin{document}

\maketitle

\vspace{-10pt}
\begin{abstract}
Causal effect estimation relies on separating the variation in the outcome into parts due to the treatment and due to the confounders.
To achieve this separation, practitioners often use
external sources of randomness that only influence the treatment called \glspl{iv}.
We study variables constructed from treatment and \gls{iv} that help estimate effects, called control functions.
We characterize general control functions for effect estimation in a meta-identification result.
Then, we show that structural assumptions on the treatment process allow the construction of general control functions, thereby guaranteeing identification.
To construct general control functions and estimate effects, we develop the \glsreset{gcfn}\gls{gcfn}.
\gls{gcfn}'s first stage called \gls{vde} constructs general control functions by recovering the residual variation in the treatment given the \gls{iv}.
Using \gls{vde}'s control function, \gls{gcfn}'s second stage estimates effects via regression.
Further, we develop semi-supervised \gls{gcfn} to construct general control functions using subsets of data that have both \gls{iv} and confounders observed as supervision; this needs no structural treatment process assumptions.
We evaluate \gls{gcfn} on low and high dimensional simulated data and on recovering the causal effect of slave export on modern community trust~\cite{nunn2011slave}.
\end{abstract}

\input{sections/introduction.tex} 
\input{sections/related.tex}
\input{sections/main-theory.tex}
\input{sections/methods.tex}
\input{sections/main-experiment.tex}
\input{sections/future.tex}

\section*{Broader Impact}
Our work applies to causal inference where strong \glspl{iv} are available to help adjust for confounding, such as in problems in healthcare and economics.
We assess the impact of our work in the context of these fields.
In general, loosening functional assumptions like \gls{gcfn} does, helps estimate effects better.
Better effect estimates help improve planning patient treatment and understanding policy impact.
However, the strong \gls{iv} assumption may not hold for all demographics.
If this occurs, demographics for which the assumption holds will have better quality effect estimates than for demographics where the assumption does not hold.
This could mean that certain demographics receive better care in hospitals or have implemented policy be more impactful on them.
Such issues could be characterized by evaluating the positivity of treatment with respect to the constructed control function in different demographics.

\section*{Acknowledgements}

The authors were partly supported by NIH/NHLBI Award R01HL148248, and by NSF Award 1922658 NRT-HDR: FUTURE Foundations, Translation, and Responsibility for Data Science.
The authors would like to thank Xintian Han and the reviewers for thoughtful feedback.

\bibliographystyle{apalike}
\bibliography{bib}
\newpage
\appendix

\onecolumn
\begin{center}
\large{Supplementary Material for \gls{vde} and \gls{gcfn}}
\end{center}

\input{sections/app-theory.tex}
\input{sections/app-exps.tex}
\printglossary[type=\acronymtype]

\end{document}

%% file: preamble/preamble.tex
\usepackage{amssymb} 
\usepackage{amsthm,amsmath}
\usepackage{graphicx}
\usepackage[labelfont=bf]{caption}

\usepackage{booktabs}

\usepackage{listings}
\usepackage{fancyvrb}
\fvset{fontsize=\normalsize}

\usepackage[colorlinks,linktoc=all]{hyperref}
\usepackage[all]{hypcap}

\usepackage[nameinlink]{cleveref}

\usepackage[acronym,smallcaps,nowarn,nohypertypes={acronym,notation}]{glossaries}

\lstdefinestyle{mystyle}{
    commentstyle=\color{OliveGreen},
    keywordstyle=\color{BurntOrange},
    numberstyle=\tiny\color{black!60},
    stringstyle=\color{MidnightBlue},
    basicstyle=\ttfamily,
    breakatwhitespace=false,
    breaklines=true,
    captionpos=b,
    keepspaces=true,
    numbers=left,
    numbersep=5pt,
    showspaces=false,
    showstringspaces=false,
    showtabs=false,
    tabsize=2
}
\usepackage{tikz}
\usetikzlibrary{shapes,decorations,arrows,calc,arrows.meta,fit,positioning}
\tikzset{
    -Latex,auto,node distance =1 cm and 1 cm,semithick,
    state/.style ={circle, draw, minimum width = 0.7 cm},
    detstate/.style ={rectangle, draw, minimum width = 0.7 cm, minimum height = 0.7 cm},
    point/.style = {circle, draw, inner sep=0.04cm,fill,node contents={}},
    bidirected/.style={Latex-Latex,dashed},
    el/.style = {inner sep=2pt, align=left, sloped}
}

\usepackage{wrapfig}
\usepackage{mathtools}

%% file: preamble/preamble_math.tex
\DeclareRobustCommand{\mb}[1]{\ensuremath{\boldsymbol{\mathbf{#1}}}}

\DeclareRobustCommand{\KL}[2]{\ensuremath{\textrm{KL}\left(#1\;\|\;#2\right)}}
\DeclareRobustCommand{\wass}[2]{\ensuremath{\mathcal{W}_1\left(#1\;\left.\right\|\;#2\right)}}
\DeclareMathOperator*{\argmax}{arg\,max}

\renewcommand{\mid}{~\vert~}

\newcommand{\mba}{\mb{a}}
\newcommand{\mbb}{\mb{b}}
\newcommand{\mbc}{\mb{c}}

\newcommand{\mbm}{\mb{m}}

\newcommand{\mbt}{\mb{t}}
\newcommand{\mbu}{\mb{u}}
\newcommand{\mbv}{\mb{v}}

\newcommand{\mbx}{\mb{x}}
\newcommand{\mby}{\mb{y}}
\newcommand{\mbz}{\mb{z}}

\newcommand{\mbI}{\mb{I}}

\newcommand{\mbdelta}{\mb{\delta}}
\newcommand{\mbepsilon}{\mb{\epsilon}}

\newcommand{\mbeta}{\mb{\eta}}

\newcommand{\supp}{\textrm{supp}}

\newcommand{\cN}{\mathcal{N}}

\newcommand{\E}{\mathbb{E}}

\newcommand{\cB}{\mathcal{B}}

\newcommand{\cU}{\mathcal{U}}

\newcommand{\hz}{\hat{\mbz}}

\newcommand{\DO}{\text{do}}

\newcommand{\g}{\mid}

\newtheorem{lemma}{Lemma}

\newtheorem{thm}{Theorem}

\crefname{lemma}{lemma}{lemmas}
\crefname{prop}{proposition}{propositions}
\newcommand\myeq{\mathrel{\overset{\makebox[0pt]{\mbox{\normalfont\tiny\sffamily c}}}{=}}}

\newcommand{\indep}{\rotatebox[origin=c]{90}{$\models$}}
\newcommand{\nindep}{\rotatebox[origin=c]{90}{$\not\models$}}

\newcommand{\Hb}{\mathbf{H}}
\newcommand{\MI}{\mathbf{I}}
\newcommand{\mbeps}{\mbepsilon}

\newcommand{\tauhat}{\hat{\tau}}

\newcommand{\ghat}{\hat{g}}

\newcommand{\kld}{\mathbf{KL}}

\newcommand{\mbzhat}{\hat{\mbz}}
\newcommand{\zhat}{\hat{z}}

%% file: preamble/preamble_acronyms.tex
\newacronym{KL}{kl}{Kullback-Leibler}
\newacronym{ELBO}{elbo}{\emph{evidence lower bound}}
\newacronym{POPELBO}{pop-elbo}{\emph{population evidence lower bound}}

\newacronym{SVI}{svi}{stochastic variational inference}
\newacronym{BUMPVI}{bump-vi}{bumping variational inference}

\newacronym{GMM}{gmm}{Gaussian mixture model}
\newacronym{LDA}{lda}{latent Dirichlet allocation}

\newacronym{SUTVA}{sutva}{stable unit treatment value assumption}

\newacronym{KSD}{ksd}{{kernelized Stein discrepancy}}
\newacronym{KCC-SD}{kcc-sd}{kernelized complete conditional Stein discrepancy}

\newacronym{OPVI}{opvi}{operator variational inference}
\newacronym{SVGD}{svgd}{Stein variational gradient descent}

\newacronym{vde}{VDE}{variational decoupling}
\newacronym{cfn}{CFN}{control function method}
\newacronym{gcfn}{GCFN}{general control function method}
\newacronym{2sls}{2SLS}{two-stage least-squares method}
\newacronym{gmm}{GMM}{generalized method of moments}
\newacronym{iv}{IV}{instrumental variable}
\newacronym{cdf}{cdf}{cumulative distribution function}

\newacronym{oosmse}{oos-mse}{counterfactual MSE}
\newacronym{oos}{oos}{out-of-sample set}

%% file: sections/introduction.tex
\vspace{-10pt}
\section{Introduction}\label{sec:intro}

Many disciplines use observational data to estimate causal effects: economics~\citep{baker1999highs}, sociology~\citep{lieberson1987making}, psychology~\citep{mcgue2010causal}, epidemiology~\citep{rothman2005causation}, and medicine~\citep{stuart2013estimating}.
Estimating causal effects with observational data requires care
due to the presence of confounders that influence both treatment and outcome.
Observational causal estimators deal with confounders in one of two ways.
One, they assume that all confounders are observed; an assumption called \emph{ignorability}.
Two, they assume a source of external randomness that has a direct influence only on the treatment.
Such a source is called an \glsreset{iv}\gls{iv}~\citep{angrist2001instrumental,heckman1985alternative}.
An example is college proximity as an \gls{iv} to study effects of education~\cite{card1993using}.

Two common \gls{iv}-based causal effect estimation methods are the \glsreset{2sls}\gls{2sls} \citep{kelejian1971two,amemiya1974nonlinear,angrist1996identification} and the traditional \glsreset{cfn}\gls{cfn}  \citep{heckman1985alternative,wooldridge2015control,wiesenfarth2014bayesian,darolles2011nonparametric}. Both methods have a common first stage: learn a distribution over the treatment conditioned on the \gls{iv}. In the second stage, \gls{2sls} regresses the outcome on simulated treatments from the first stage, while \gls{cfn}'s second stage regresses the outcome on the true treatment and the error in the prediction of treatment from the first stage.
The prediction error can be used to control for confounding and is
thus called a \emph{control function}.
Though widely used,
both \gls{2sls} and \gls{cfn} breakdown under certain conditions like, for example,
when the outcome depends on multiplicative interactions of treatment and confounders.
Further, \gls{cfn} requires an additional assumption about the correlations between noise and outcome.

We study causal estimation with control functions.
To estimate effects, control functions must satisfy ignorability.
Our meta-identification result (\cref{thm:intro}) shows that a control function satisfies ignorability if 1) the control function and \gls{iv} together reconstruct the treatment, and 2) the confounder and control function together are jointly independent of the \gls{iv}.
We will refer to such control functions as \textit{general control functions}.
Effect estimation in general requires that the treatment has a chance to take any value
given the control function; this is called positivity.
We show positivity for general control functions holds if the \gls{iv} can set treatment to any value; we call this a \textit{strong} \gls{iv}.

Any general control function uniquely determines the effect because it satisfies ignorability and positivity (given a strong \gls{iv}).
Causal identification requires effects to be uniquely determined by the observed data distribution.
Thus, building general control functions using observed data guarantees causal identification.
As reconstruction and \textit{marginal} independence are properties of the joint distribution over observed data and control function, they can be guaranteed.
Guaranteeing \textit{joint} independence requires further assumptions as it involves the \textit{unobserved} confounder.
We show that structural assumptions on the treatment process, such as treatment being an additive function of the confounder and \gls{iv}, help ensure joint independence. 

To build general control functions and use them to estimate effects, we develop the \glsreset{gcfn}\gls{gcfn}.
\gls{gcfn}'s first stage, called \glsreset{vde}\gls{vde}, constructs the general control function.
\gls{vde} is a type of autoencoder where the encoder constructs the control function and the decoder reconstructs treatment from control function and \gls{iv}, under the constraint that the control function and \gls{iv} are independent.
When \gls{vde} is perfectly solved with a decoder that reflects a structural treatment process assumption, like additivity, reconstruction and joint independence are guaranteed.
Thus with a strong~\gls{iv}, ignorability and positivity hold which implies identification, and that effect estimation does not require structural assumptions on the \textit{outcome process} like those in \gls{2sls} and \gls{cfn}.
Using \gls{vde}'s general control function, \gls{gcfn}'s second stage estimates the causal effect.
\gls{gcfn}'s second stage can be any method that relies on ignorability 
like matching/balancing methods~\citep{wager2018estimation,dehejia2002propensity,shalit2017estimating} and doubly-robust methods~\citep{funk2011doubly}.

We also consider a setting where a subset of the data has observed confounders that provide ignorability.
We develop semi-supervised \gls{gcfn} to estimate effects in this setting.
Semi-supervised \gls{gcfn}'s first stage is an augmented \gls{vde} that forces the control function to match the confounder in the subset where it is observed.
This augmented \gls{vde} helps guarantee joint independence even with a decoder that does not reflect structural treatment process assumptions.

In~\cref{sec:exps}, we evaluate \gls{gcfn}'s causal effect estimation on simulated data with the outcome, treatment, and \gls{iv} observed.
We demonstrate how \gls{gcfn} produces correct effect estimates without additional assumptions on the true outcome process, whereas \gls{2sls}, \gls{cfn}, and DeepIV \citep{DBLP:conf/icml/HartfordLLT17} fail to produce the correct estimate.
Further, we show that \gls{gcfn} performs on par with recently proposed methods DeepGMM~\citep{bennett2019deep} and DeepIV~\citep{DBLP:conf/icml/HartfordLLT17} on high-dimensional simulations from each respective paper.
We also demonstrate that in data with a small subset having observed confounders, 
semi-supervised \gls{gcfn} outperforms outcome regression on treatment and confounder within the subset.
We also show recovery of the effect of slave export on current societal trust~\cite{nunn2011slave}.

%% file: sections/related.tex
\vspace{-7pt}
\paragraph{Related Work.}\label{sec:related}
Classical examples of methods that use \glspl{iv} include the Wald estimator \citep{wald1940fitting}, \glsreset{2sls}\gls{2sls} \citep{amemiya1974nonlinear,angrist1996identification,kelejian1971two} and \glsreset{cfn}\gls{cfn}  \citep{darolles2011nonparametric,heckman1985alternative,wooldridge2015control,wiesenfarth2014bayesian}.
The Wald estimator assumes constant treatment effect.
\gls{2sls}'s estimation could be biased when the outcome generating process has multiplicative interactions between treatment and confounders (\cref{app:subsec-2sls}).
\citet{guo2016control} proved that under some assumptions, \gls{cfn} improves upon \gls{2sls}.
Beyond these classical estimators, \citet{wooldridge2015control} discusses extensions of regression residuals for non-linear models under distributional assumptions about the noise in the treatment process.
\citet{DBLP:conf/icml/HartfordLLT17} developed DeepIV, a deep variant of \gls{2sls} and \citet{singh2019kernel} kernelized the \gls{2sls} algorithm.
An alternative to \gls{2sls} is the \gls{gmm}~\citep{hansen1982large} which solves moment equations implied by the independence of the confounder and the \gls{iv}.
\citet{bennett2019deep} develop a minimax \gls{gmm} and use neural networks to specify moment conditions.

Given only an \gls{iv}, treatment, and outcome, causal effects are not identifiable without further assumptions \citep{balke1997bounds,manski1990nonparametric}.
\citet{newey2013nonparametric} and \citet{chetverikov2017nonparametric} assume additive outcome processes, where the outcome process is a sum of the causal effect and zero-mean noise; such models are also called separable.
Identification in separable models relies on the \emph{completeness} condition~\citep{cunha2010estimating} which requires the conditional distribution of treatment given \gls{iv} to sufficiently vary with the \gls{iv}.
\citet{newey2013nonparametric,chetverikov2017nonparametric} discuss non-parametric estimators under assumptions of monotonicity of the treatment process and shape of causal effects (for eg. $U$-shaped).
We focus on the setting where the outcome process \textit{cannot} be represented as a sum of the causal effect and noise, often called a non-separable model~\cite{chesher2003identification}.
\citet{imbens2009identification} showed effect identification in non-separable models when the treatment has a continuous strictly monotonic \gls{cdf} given the \gls{iv}.
Under this same condition, we can guarantee joint independence via a strictly monotonic reconstruction map which means identification holds.

\subsection{Review of \glspl{iv} and traditional control function theory}\label{sec:review}
\vspace{-5pt}
\begin{wrapfigure}[7]{r}{0.3\textwidth}
  \vspace{-30pt}
      \begin{center}
        \begin{tikzpicture}
        \node[state,fill=gray] (z) at (2.5,0.7) {$\mbz$};
        \node[state] (e) at (0,0) {$\mbepsilon$};
        \node[state] (t) [right =of e] {$\mbt$};

        \node[state] (y) [right =of t] {$\mby$};

        \path (e) edge (t);
        \path[dashed] (z) edge (t);
        \path (t) edge (y);
        \path[dashed] (z) edge (y);
        \path[dashed] (z) edge (t);
    \end{tikzpicture}
      \end{center}
      \caption{Causal graph with hidden confounder $\mbz$, outcome $\mby$, instrument $\mbepsilon$, treatment $\mbt$. }\label{fig:iv-graph}
\end{wrapfigure}
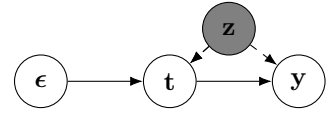  
To define the causal effect we use causal graphs~\citep{pearl2009causal}.
In causal graphs, each variable is represented by a node, and each causal relationship is a directed arrow from the cause to the effect.
Causal graphs get transformed by interventions with the $\DO$-operator.
The shared relationships between the graphs
before and after the $\DO$-operation make estimation possible.
The causal effect of giving a treatment $\mbt=a$ on an outcome $\mby$ is $\E[\mby \g \DO(\mbt=a)]$.
The causal graph in \Cref{fig:iv-graph} describes a broad class of \gls{iv} problems.
The difficulty of causal estimation in this graph stems from the unobserved confounder $\mbz$. The \gls{iv} $\mbepsilon$ helps control for $\mbz$.
Two popular \gls{iv}-based methods are the \glsreset{2sls}\gls{2sls} and \glsreset{cfn}\gls{cfn}.

We follow  the \gls{cfn} setup from \citet{guo2016control}, where the true outcome and treatment processes have additive zero-mean noise called $\mbeta_{\mby}$ and $\mbeta_{\mbt}$ that may be correlated due to $\mbz$:
\begin{align}\label{eq:linear-confounding}
 \mby = f(\mbt) + \mbeta_{\mby}, \quad \mbt = g(\mbepsilon) + \mbeta_{\mbt}.
\end{align}
To estimate the causal effect, the \gls{cfn} method constructs a control function with the regression residual $\mbt - \hat{g}(\mbepsilon)$.
Then, \gls{cfn} regresses the outcome $\mby$ on the regression residual and the treatment $\mbt$.
The causal effect is the estimate of the function $f(\mbt)$.
For this estimate to be valid, the \gls{cfn} method assumes that $\mbeta_{\mbt},\mbeta_{\mby}$ satisfy the following property for some constant $\rho$, (\text{assumption A4 in [\citenum{guo2016control}]}):
\begin{align}\label{eq:linear-noise-assumption}
& \E[\mbeta_{\mby}\g  \mbeta_{\mbt}=\eta] = \rho \eta
\end{align}

This property restricts the applicability of the \gls{cfn} method by limiting how confounders influence the outcome and the treatment. Consider the following additive noise example:
$ \,\, \mbepsilon, \mbz\sim \cN(0,1),\,\, \mbt = \mbz + \mbepsilon,\,\, \mby \sim \cN(\mbt^2 + \mbz^2,1 )$, where $\cN$ is the standard normal.
Here $\mbeta_{y}= \mbz^2$ and $\mbeta_{\mbt} = \mbz$ meaning that $\E[\mbeta_{\mby}\g \mbeta_{\mbt}=\eta] = \eta^2$, violating the assumption in \cref{eq:linear-noise-assumption}.
Note that  $\E[\mbz\mbt^2] = \E[\mbz\mbz^2] = 0$, however $\E[\mbt^2\mbz^2]>0$.
This means regressing $\mby$ on $\mbt^2$ and $\mbz$, i.e., with the correct model for $f(\mbt)$, would result in an inflated coefficient of $\mbt^2$, which is an incorrect causal estimate.
\Cref{eq:linear-noise-assumption} is required because some specified function of $\mbt$ could be correlated with an unspecified function of $\mbz$, resulting in a biased causal estimate.
See \cref{app:subsec-2sls} for an example where \gls{2sls} produces biased effect estimates.
The assumption in \cref{eq:linear-confounding} restricts the confounder's
influence to be additive on both the treatment and outcome.
Further, \gls{cfn} assumes that the average additive influence the confounder has on the outcome to be a scaled version of the confounder's influence on the treatment (\cref{eq:linear-noise-assumption}).
Such assumptions may not hold in
real data. For example, the effect of a medical treatment on patient lifespan
is confounded by the patient's current health. This confounder influences the treatment through a human decision process, while it influences the outcome through
a physiological process making it unlikely to meet \gls{cfn}'s assumptions.

%% file: sections/main-theory.tex
\section{Causal Identification with General Control Functions}
With a control function that satisfies ignorability and positivity, causal estimation reduces to regression of the outcome on the treatment and the control function.
We characterize such control functions:
\newcommand\theoremone{
Let $F(\mbt,\mbepsilon,\mby)$ be the true data distribution.
Let control function $\hz$ be sampled conditionally on $\mbt,\mbepsilon$.
Let $q(\hz, \mbt, \mbepsilon) = q(\hz \g \mbt, \mbepsilon) F(\mbt, \mbepsilon)$ be the joint distribution over $\hz, \mbt, \mbepsilon$.
Further, let $g$ be a deterministic function and $\mbdelta$ be independent noise such that \ $\mbt = g(\mbz,\mbepsilon, \mbdelta)$ and let the implied true joint be $F'(\mbt,\mbz,\mbdelta)$. 
Assume the following:
\vspace{-5pt}
 \begin{enumerate}[leftmargin=*,itemsep=-2pt]
      \item (A1) $\hz$ satisfies 
        the \textbf{reconstruction} property: $\exists d,\, \hz,\mbt, \mbepsilon \sim q(\hz, \mbt, \mbepsilon)  \implies \mbt = d(\hz,\mbepsilon)$.
     \item (A2) The \gls{iv} is \textbf{jointly independent} of control function, true confounder, and noise $\mbdelta$: $\mbepsilon \indep (\mbz,\hz, \mbdelta)$.
     \item (A3) \textbf{Strong} \gls{iv}.
     For any compact $B\subseteq \supp(\mbt)$, 
     $\exists c_B\, \textrm{s.t.\ a.e.\ } t\in B$, $F'(\mbt =t \g \mbz, \mbdelta) \geq c_B>0.$
 \end{enumerate}
Then, the control function $\hz$ satisfies ignorability and positivity:
\begin{align*}
  & q(\mby \g \mbt=t, \hz )= q(\mby \g \emph{\DO} (\mbt=t), \hz)
  \quad \quad \emph{\text{a.e. in }} \supp(\mbt) \quad q(\hz) >0 \implies q(\mbt=t\g  \hz)>0.
\end{align*} 
Therefore, the true causal effect is uniquely determined by $q(\hz, \mbt, \mby)$ for almost every $t\in \supp(\mbt)$:
\begin{align*}
    \E_{\hz}[ \mby \g \mbt=t, \hz] & = \E_{\hz}[ \mby \g \emph{\DO}(\mbt=t), \hz]
    = \E[\mby\g \emph{\DO}(\mbt=t)].
\end{align*}}

\begin{thm}\label{thm:intro} \emph{\textbf{(Meta-identification result for control functions)}}\\
\theoremone{}
\end{thm}
\Cref{thm:intro} characterizes functions of treatment and \gls{iv} that satisfy reconstruction (A1) and  joint independence (A2) which we call \textit{general control functions}.
Positivity of $\mbt$ w.r.t.\ the general control function holds under an assumption about the treatment process that the \gls{iv} is strong (A3).
Ignorability and positivity w.r.t.\ $\hz$ imply that the true causal effect is uniquely determined as a function of the observed data distribution $q(\hz, \mbt, \mby)$
\footnote{We also require that $\E_{\hz}\E[\mby\g \DO(\mbt), \hz]$ 
exists.
This is guaranteed if the causal effect $\E[\mby\g \DO(\mbt)] = \E_{\mbz}\E[\mby \g \mbt, \mbz]$ exists as ignorability holds w.r.t.\ $\hz$: $\E_{\hz}\E[\mby\g \DO(\mbt), \hz] = \E_{\hz}\E_{q(\mbz \g \hz)}\E[\mby\g \mbt, \mbz] = \E_{\mbz}\E[\mby \g \mbt, \mbz]$.
}.
If A1 and A2 are satisfied by the observed data distribution $q(\hz, \mbt, \mby, \mbeps)$, the true effect is uniquely determined by the observed data distribution and thus causal identification holds.
However, joint independence (A2) relies on the \textit{unobserved} true confounder $\mbz$.
So, \cref{thm:intro} is a \textit{meta}-identification result because it does not specify how to guarantee joint independence using $q(\hz, \mbt, \mbeps)$.
In~\cref{sec:identification}, we discuss structural assumptions on the treatment process that instantiate this meta-result and guarantee identification.

\Cref{thm:intro} holds for both discrete and continuous $\mbt$ given that the causal effect exists for all $t\in \supp(\mbt)$.\footnote{Effects for certain treatments can be identified even without the strong \gls{iv} assumption (A3):
for any compact subset $B \subseteq \supp(\mbt)$ such that $\forall t \in B,\,\, F^\prime(\mbt =t \g \mbz, \mbdelta)\geq c_B>0$, effects can be estimated for all $t\in B$\label{pos-viol}.\vspace{-15pt}}
While we focus on the causal effect $\E[\mby \g \DO(\mbt)]$, \cref{thm:intro} guarantees any property of $\mby\g \DO(\mbt)$ can be estimated; for e.g. quantile treatment effects.
For ease of exposition, we restrict ourselves to treatments of the form $\mbt = g(\mbeps, \mbz)$, without noise $\mbdelta$.
Then, \cref{thm:intro} requires only $\mbeps \indep (\mbz, \hz)$.
In~\cref{appsec:general-t-process}, we show $\mbeps \indep (\hz, \mbz, \mbdelta)$ is guaranteed for more general treatment processes of the form $\mbt= g(\mbeps, h(\mbz, \mbdelta))$.
Guaranteeing joint independence requires further conditions and is the central challenge in developing two-stage \gls{iv}-based estimators.

\paragraph{Why joint independence?} 
A potential outcome $\mby_{\mbt}$ is the outcome that would be observed if a unit is given treatment $\mbt$.
The potential outcome $\mby_{\mbt}$ follows the distribution of $\mby$ under the $do$ operator and only depends on the true confounder $\mbz$.
For ignorability with respect to $\hz$, we need $\mby_{\mbt}$ to be independent of $\mbt$, given $\hz$.
By reconstruction, given $\hz$, $\mbt$ is purely a function of $\mbepsilon$.
This means ignorability with respect to the control function $\hz$ requires that the true confounder and \gls{iv} be independent given the control function.
Therefore, ignorability
requires $\mbz \indep \mbeps \g \hz$.
Further, conditional independence $\hz\indep \mbeps\g \mbz$ implies positivity of $\mbt$ w.r.t\ $\hz$ if $\mbeps$ is strong. 
Joint independence $\mbepsilon \indep (\mbz, \hz)$ implies both the conditional independencies above.

The causal graph \cref{fig:iv-graph} with $\mby$ marginalized out can
be represented with two sources of randomness one from the unobserved
confounder $\mbz$ and one from the \gls{iv} $\mbeps$; the extra randomness in $\mbt$
denoted as $\mbdelta$ can be absorbed into $\mbz$. In this setup, the
treatment and control function are deterministic functions of
the unobserved confounder and \gls{iv}.
With only two sources of randomness, joint independence means
the control function $\hz$ needs to
only be a function of the true unobserved confounder $\mbz$.
When $\hz$ is a stochastic function of the treatment and \gls{iv}, joint independence holds if $\hz$ determines $\mbz$ while $\hz\indep \mbeps$.

As $\hz$ and $\mbeps$ are observed, we can guarantee $\hz \indep \mbeps$.
The marginal independence $\mbz\indep\mbeps$ holds by definition of an \gls{iv}. 
However, even both marginal independencies $\hz\indep \mbepsilon$ 
and $\mbz \indep \mbepsilon$ together do not
imply joint independence $\mbepsilon \indep (\hz,\mbz)$. This means 
a control function $\hz$ that satisfies the reconstruction property and marginal independence $\hz \indep\mbepsilon$ may fail to yield ignorability.
In~\cref{app:marginal-v-joint}, we build an example of a deterministic almost everywhere invertible function
of two independent variables $\mbc = f(\mba, \mbb)$ such that $\mbc\indep\mba$ and $\mbc\indep \mbb$ and yet, joint independence $ (\mbc, \mbb)\nindep \mba$ is violated.
As $\mbz$ is unobserved, achieving joint independence requires further assumptions.
Next, we discuss how structural assumptions on the true treatment process can help guarantee joint independence.

\subsection{Guaranteeing joint independence for identification}\label{sec:identification}
We show structural treatment process assumptions help guarantee joint independence by relating it to $q(\hz, \mbt,\mbeps)$ and thus giving identification.
Joint independence can be guaranteed (via marginal independence) if the reconstruction map $d(\hz, \mbeps)$ (A1, \cref{thm:intro}) reflects the functional structure of the treatment process.
As an example, consider an additive treatment process $\mbt = \mbz + g(\mbepsilon)$.
If the reconstruction map is $d(\hz, \mbeps) = h^\prime(\hz) + g^\prime(\mbeps)$ and $\mbeps \indep \hz$, joint independence holds.
To see this, note
\begin{align}
  h^\prime(\hz) - \E_{\hz}[h^\prime(\hz)] = \mbt - \E[\mbt\g \mbepsilon] = \mbz - \E_{\mbz}[\mbz] \implies \exists \text{ constant  } c,\, h^\prime(\hz) = \mbz + c,
\end{align}
 meaning $h^\prime(\hz)$ determines $\mbz$.
By $\hz\indep \mbepsilon$, it holds that
$
q(\hz,\mbz \g \mbepsilon) =  q(\hz , h^\prime(\hz) - c \g \mbepsilon)  
                          = q(\hz,\mbz).
                          $
Thus, leveraging the functional structure of the treatment process helps guarantee joint independence by relating it to $q(\hz, \mbt, \mbeps)$, via $\hz \indep \mbeps$.
Assuming treatment gets generated from other \emph{known} invertible functions, such as multiplication $\mbt = h(\mbz) * g( \mbepsilon)$, also leads to joint independence.
\citet{imbens2009identification} proved effect identification when
the treatment is a continuous strictly monotonic function of the confounder; these conditions helps guarantee joint independence (see~\cref{appsubsec:monotonicity}).
For more general treatments of the form $\mbt = g(\mbeps, h(\mbz, \mbdelta))$
the structural assumptions from above can only guarantee $(h(\mbz, \mbdelta), \hz)\indep \mbeps$; see \cref{appsec:add-dec} for general additive treatments: $\mbt = h(\mbz,\mbdelta) + g(\mbepsilon)$. 
However, we show in~\cref{appsec:general-t-process} that for such general treatment processes $(h(\mbz, \mbdelta), \hz)\indep \mbeps \implies (\mbz, \hz, \mbdelta) \indep \mbeps$ which, together with reconstruction, implies ignorability(\cref{thm:intro}).
In summary, under certain structural assumptions, general control functions exist ($\hz = \mbz$ for example) and can be built using only properties of the observed data distribution $q(\hz, \mbt, \mbeps)$. This guarantees identification.
In~\cref{sec:estimation}, we develop practical algorithms to build general control functions.

\subsection{Comparison of identification with general control functions to existing work}
Traditional \gls{cfn} theory~\cite{guo2016control} relies on the assumption that the treatment process is additive; recall $\mbt = g(\mbepsilon) + \mbeta_t$ from~\cref{sec:review} and $\mbeta_t$ is correlated with outcome noise due to $\mbz$.
Beyond this additivity assumption, traditional \gls{cfn} theory further assumes 
1) the outcome process is additive, like in~\cref{eq:linear-confounding},  
2) the noise $\mbeta_y$ in the outcome process is independent of the \gls{iv},
3) linear noise relationship between $\mbeta_t, \mbeta_y$, like in~\cref{eq:linear-noise-assumption}, and 
4) (relevance) the treatment effect function and \gls{iv} are correlated~\cite{guo2016control}.
When the treatment process is additive, joint independence can be guaranteed as a property of the distribution $q(\hz, \mbt,\mbeps)$, via $\hz \indep \mbeps$; see~\cref{sec:identification}.
Then, identification with general control functions requires a strong \gls{iv}.
While it allows structural outcome process assumptions (like 3) can be relaxed, a strong \gls{iv} needs more than the two \gls{iv} properties, independence with confounder and relevance.
However, domain expertise helps reason about strong \glspl{iv}; for example, can college proximity influence a student’s decision to go to college regardless of skill? If yes, college proximity is a strong \gls{iv}.
We compare against other identification conditions (like \gls{2sls} and~\cite{imbens2009identification}) in~\cref{appsec:identification}.

%% file: sections/methods.tex
\section{The General Control Function Method (\gls{gcfn})}\label{sec:estimation}
\gls{gcfn} constructs a general control function and estimates effects with it.
\gls{gcfn} has two stages.
The first stage constructs a general control function as the code of an
autoencoder.
The second stage builds a model from the control function and the treatment to the outcome and estimates effects.
\vspace{-5pt}
\paragraph{Variational Decoupling}\label{subsec:vde}
We construct the control function $\hz$ as a stochastic function of the treatment $\mbt$ and the \gls{iv} $\mbeps$; with parameter $\theta$, the estimator is $q_\theta(\hz \g  \mbt,\mbepsilon)$.
First, to guarantee the reconstruction property (A1 in~\cref{thm:intro}), the control function and the \gls{iv} must determine treatment, implying that with parameter $\phi$, $p_\phi(\mbt\g \hz=\zhat,\mbepsilon)$ should be maximized for $\zhat\sim q(\hz\g \mbt,\mbepsilon)$.
Together, these form the parts of an autoencoder where a control function is sampled conditioned on the treatment and \gls{iv}, while the treatment is reconstructed from the same control function and \gls{iv}.
Second, to guarantee marginal independence, we force the control function to be independent of the \gls{iv}: $\hz\indep \mbepsilon$.
Let the true data distribution be $F(\mbt,\mbepsilon)$ and $\mathbf{I}$ denote mutual information.
Putting the two parts together, we define a constrained optimization to construct $\hz$, called \glsreset{vde}\gls{vde}:
\begin{align}\label{eq:confest}
  \begin{split}
    \text{(\gls{vde})} \quad \max_{\theta,\phi} & \, \E_{F(\mbt,\mbepsilon)} \E_{ q_{\theta }(\hz\g  \mbt,\mbepsilon)}\log p_{\phi }(\mbt \g \hz, \mbepsilon) \quad s.t\quad \mathbf{I}_\theta(\hz; \mbepsilon) = 0.
  \end{split}
\end{align}
Recall from~\cref{sec:identification} that with a reconstruction map $d(\hz, \mbeps)$ (from A1 in~\cref{thm:intro}) that reflects the functional structure of the treatment, marginal independence $\hz \indep \mbeps$ implies joint independence.
To model such a map, \gls{vde}'s decoder, $p_\phi(\mbt\g \hz=\zhat,\mbepsilon)$ reflects the same functional structure.
For example, with an additive treatment process the decoder would be parametrized as $ \log p_\phi(\mbt \g \hz, \mbeps) \propto  -(\mbt - h_\phi^\prime(\hz) - g_\phi^\prime(\mbeps))/\sigma_\phi^2$; $\sigma_\phi$ allows for a point-mass distribution $p_\phi$ at optimum of \gls{vde}.
In summary, beyond the observed treatment and \gls{iv}, \gls{vde} takes a specification of the functional structure of the treatment process as input which informs the structure of the decoder.

\gls{vde} is converted to an unconstrained optimization problem by absorbing the independence constraint into the optimization via the Lagrange multipliers trick with $\lambda>0$,
\begin{align}\label{eq:unconstrained-vde}
\max_{\theta,\phi}&\,  \E_{F(\mbt,\mbepsilon)}\E_{q_{\theta }(\hz\g \mbt,\mbepsilon)}\log p_{\phi }(\mbt\g \hz, \mbepsilon)  - \lambda\mathbf{I}_\theta(\hz; \mbepsilon).
\end{align}
Estimation of the mutual information requires $q_{\theta}(\hz \g \mbepsilon)$.
Instead, we lower bound the negative mutual information by introducing an auxiliary distribution $r_{\nu}(\hz)$. This yields a tractable objective:
\begin{align}\label{eq:lowerbound}
  \begin{split}
    \max_{\theta,\phi,\nu} \,  &\E_{F(\mbt,\mbepsilon)} \big[(1 + \lambda) \E_{q_{\theta }(\hz\g \mbt,\mbepsilon)}\log p_{\phi }(\mbt\g \hz, \mbepsilon)
    -\lambda \KL{q_{\theta}(\hz \g  \mbt, \mbepsilon)}{r_{\nu}(\hz)}\big].
  \end{split}
\end{align}
A full derivation can be found in \Cref{appsec:milb}.
The lower bound is tight when the auxiliary distribution $r_\nu(\hz) = q_\theta(\hz)$.
For example, when $q_\theta(\hz \g  \mbt,\mbepsilon)$ is categorical, optimizing \cref{eq:lowerbound} with a categorical $r_\nu(\hz)$ makes the lower bound tight.
The parameters $\theta, \phi, \nu$ can be learned via stochastic optimization.
\gls{vde} can be adapted to use covariates by conditioning on the covariates as needed.
\vspace{-5pt}

\paragraph{Outcome Modeling.}
\gls{vde} provides a general control function $\hz$ and its marginal distribution $q_\theta(\hz)$.
If the \gls{iv} is strong, $\hz$ satisfies ignorability and positivity and the causal effect can be estimated by regressing the outcome on the control function and the treatment.
Other effect estimation methods like matching/balancing methods~\citep{wager2018estimation,dehejia2002propensity,shalit2017estimating} and doubly-robust methods~\citep{funk2011doubly} can be used.
This regression is \gls{gcfn}'s second stage, called the outcome stage.
We formalize this outcome stage as a maximum-likelihood problem and learn a model with parameters $\beta$ under the true data distribution $F(\mby,\mbt,\mbepsilon)$ 
and the general control function distribution $q_{\theta}(\hz\g \mbt,\mbepsilon)$:
\begin{align}
\arg\max_\beta\E_{F(\mby,\mbt,\mbepsilon)}\E_{q_\theta(\hz \g \mbt,\mbepsilon)}\log p_\beta\left( \mby \g \hz,\mbt \right).
\label{eq:outcome}
\end{align}
\vspace{-5pt}

\paragraph{Semi-Supervised GCFN.}
The explicit optimization to learn the control function
in \gls{vde} makes it simple to take advantage
of datapoints where both the confounder and \gls{iv} are observed by forcing the control function
to predict the observed confounder.
Let $\mbm$ be an missingness indicator variable that is $1$ when the true confounder $\mbz$ is observed and $0$ otherwise.
Let the joint distribution be $F(\mbt, \mbepsilon,\mbm, \mbz)$
and $\zeta$ be a scaling hyperparameter parameter.
Then the augmented \gls{vde} stage in semi-supervised \gls{gcfn}, with $\kappa=\nicefrac{\lambda}{(1 + \lambda)}$, is
\begin{align}\label{eq:semi-gcfn}
  \begin{split}
    \max_{\theta,\phi,\nu} \,  &\mathop{\E}_{F(\mbt,\mbepsilon, \mbm, \mbz)} \left[\mathop{\E}_{q_{\theta }(\hz\g \mbt,\mbepsilon)}\log p_{\phi }(\mbt\g \hz, \mbepsilon)
    -\kappa \KL{q_{\theta}(\hz \g  \mbt, \mbepsilon)}{r_{\nu}(\hz)} 
     + \zeta \mbm\log q_\theta(\hz = \mbz \g \mbt,\mbepsilon).\right]
  \end{split}
\end{align}
The added term $\log q_\theta(\hz = \mbz \g \mbt,\mbepsilon)$ encourages
the control function to place all of its mass on the observed confounder value.
When the control function places all of its
mass on the confounder, the control
function is determined by value of the confounder.
Together with the fact that the 
confounder is independent of the \gls{iv}, this implies
the control function, confounder
pair is jointly independent of the instrument.
Therefore, given enough datapoints with the confounder and \gls{iv} observed, joint independence can be guaranteed without treatment assumptions like in~\cref{sec:identification}.
The second stage of semi-supervised \gls{gcfn}
uses the outcome regression in \cref{eq:outcome}
to estimate effects.

\subsection{Error bounds for \gls{gcfn}'s estimated effects}
An imperfectly estimated general control function may violate the conditional independence $\mbz\indep \mbepsilon \g \hz$ which is required for ignorability.
If ignorability does not hold, estimated effects are biased.
First, assuming an additive treatment process, we bound the expected bias in causal effects using quantities optimized during training in \gls{vde}, specifically reconstruction error and dependence of $\hz$ on $\mbeps$:
\newcommand{\addprocthm}{
  \begin{thm}
    Assume an additive treatment process $\mbt= \mbz + g(\mbeps)$ where $g$ is an $L_g$-Lipschitz function, and $\E_{F(\mbz)}\mbz=0$. Let $\E[\mby\g \mbt=t, \mbz=z] = f(t,z)$ be an $L$-Lipschitz function in $z$ for any $t$. Further,
    \vspace{-15pt}
      \begin{enumerate}[leftmargin=15pt,itemsep=-3pt]
        \item 
        let reconstruction error be non-zero but bounded 
        $\E_{q(\mbt, \hz, \mbeps)}(\mbt - \hz - g^\prime(\mbeps))^2 \leq \delta.$
        Assume that $g^\prime$ is also $L_g$-Lipschitz.
        Further, let $\E_{q(\hz)}\hz=0$, and $\E_{q(\hz)}|\hz|<\infty$.
        \item Assume 
        $\mbeps\nindep\hz$ and let the dependence be bounded: 
        $\max_{\zhat}\wass{q(\mbeps\g\hz=\zhat)}{F(\mbeps)}\leq \gamma$.
      \end{enumerate}
      \vspace{-5pt}
      With the estimated and true causal effects as $\tauhat(t) = \E_{\hz} f(t, \hz)$ and $\tau(t) = E_{\mbz}f(t, \mbz)$ respectively,
    \[
      \,\, \E_{F(\mbt)}|\tauhat(\mbt) - \tau(\mbt)| \leq L \sqrt{\delta + 4\gamma L_g \E_{q(\hz)}|\hz|}
      .
      \]
  \end{thm}}
\addprocthm{}
See~\cref{appsec:add-t-error-bound} for the proof.
Second, in~\cref{thm:bound-thm} in~\cref{appsec:gen-err-bound}, we prove a general error bound for \gls{gcfn} that depends on the residual confounding that $\hz$ does not control for, measured as the conditional mutual information $\MI(\mbz; \mbt\g \hz)$.
When $\MI(\mbz; \mbt\g \hz) > 0$, ignorability may not hold 
and estimated effects are biased.
Assuming positivity and a sufficiently concentrated $\mbz\g \hz$, we prove in~\cref{thm:bound-thm} that $\MI(\mbz; \mbt\g \hz)$ controls average absolute error in effects.
This error is tempered by the smoothness of outcome as a function of the confounder $\mbz$.
This bound also accounts for errors due to poor estimation of $\E[\mby\g \mbt, \hz]$ in low density regions of $q(\mbt, \hz)$ which may occur when $\hz\nindep\mbeps$.

%% file: sections/main-experiment.tex
\section{Experiments}\label{sec:exps}
We evaluate \gls{gcfn} on simulated data, where the true causal effects are known and show that \gls{gcfn} corrects for confounding and estimates causal effects better than \gls{cfn}, \gls{2sls}, and a \gls{2sls} variant, DeepIV~\cite{DBLP:conf/icml/HartfordLLT17}.
We then evaluate \gls{gcfn} on high-dimensional data using simulations from DeepIV~\cite{DBLP:conf/icml/HartfordLLT17} and DeepGMM~\cite{bennett2019deep}.
Then, we estimate the effect of slave export on community trust~\cite{nunn2011slave} and compare \gls{gcfn}'s estimate to the effect reported in~\citep{nunn2011slave}.

\paragraph{Experimental details}
For \gls{gcfn}, we let the control function $\hz$ be a categorical variable.
The encoder in \gls{vde}, $f_\theta$, is a 2-hidden-layer neural network $f_\theta$, which parametrizes a categorical likelihood $q_\theta(\hz = i \g \mbt=t,\mbepsilon=\epsilon)\propto \exp{(f_{\theta}(t,\epsilon, i))}$.
The decoder is also a 2-hidden-layer network; the reconstructed likelihood of ${\mbt}$ is different for different experiments.
In all experiments, the hidden layers in both encoder and decoder networks have $100$ units and use ReLU activations.
The outcome model is also a 2-hidden-layer neural network with ReLU activations.
For the simulated data, the hidden layers in the outcome model have $50$ hidden units.
In estimating the effect of slave export, the hidden layers in the outcome model have only $10$ hidden units; larger width resulted in overfitting.
Unless specified otherwise, we train on $5000$ samples with a batch size of $500$ for optimizing both \gls{vde} and the outcome model for $100$ epochs with Adam~\cite{kingma2014adam}.
In \cref{sec:decoder-structure} and~\cref{sec:partial-obs}, we evaluate effect estimates on a subset of the support of the treatment distribution where the most mass lies: $200$ equally spaced treatment values in $[-1,1]$.
We defer other details to~\cref{appsec:exps}.

All hyperparameters for \gls{vde}, except the mutual-information coefficient $\kappa=\lambda/(1 + \lambda)$, and the outcome-stage were found by evaluating the respective objectives on a held-out validation set.
In our experiments, we found that setting  $\kappa$ between $0.1-0.4$ worked best.
\gls{gcfn}'s performance was only mildly sensitive to changing $\kappa$ within this range.
However, one can tune $\kappa$ further by choosing the one which gives the control function $\hz_{\kappa}$ that results in the largest expected outcome likelihood on a heldout set.
This procedure relies on \gls{vde} and outcome objectives reaching optimum if and only if $\hz$ satisfies perfect reconstruction and marginal independence.
See~\cref{appsec:val} for further details.

\subsection{Simulations with specific decoder structure}\label{sec:decoder-structure}
We compare \gls{gcfn}'s performance against \gls{2sls}, \gls{cfn} and DeepIV and show that \gls{gcfn} outperforms these methods when the functional properties of the treatment process are known.
We consider two settings with continuous outcome, treatment, and confounders where the assumptions of \gls{2sls} and \gls{cfn} fail: 1) with an additive treatment process and a multiplicative outcome process and 2) with a multiplicative treatment process and an additive outcome process.
For both settings, the causal effect is the same $\E[\mby\g \DO(\mbt=t)] = t$.
The control function $\hz$ is set to have $50$ categories.
We report results for the mutual information coefficient $\kappa= \nicefrac{\lambda}{1+\lambda} = 0.1$.
We consider $3$ different strengths of confounding as captured by the parameter $\alpha\in [0.5,1.0,2.0]$.
\begin{figure}[t]
  \begin{minipage}{.3\textwidth}
    \centering
    \includegraphics[width=\columnwidth]{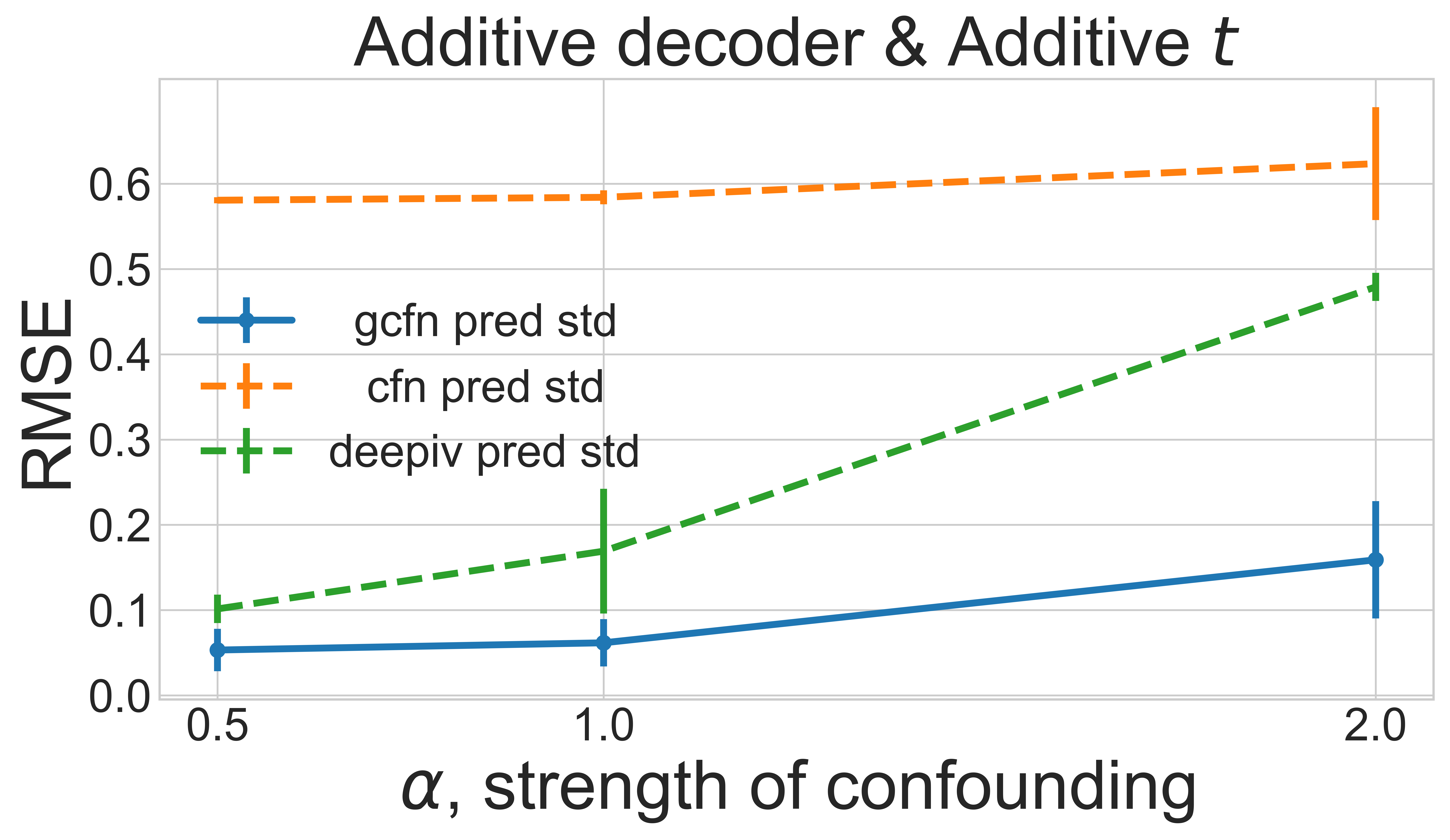}
    \caption{\gls{gcfn} obtains better effect estimates than \gls{cfn} and DeepIV when the \emph{additive outcome process} assumption is violated.
    }\label{fig:exps}
  \end{minipage}
  \hspace{0.2cm}
  \begin{minipage}{0.3\textwidth}
    \centering
    \includegraphics[width=1\columnwidth]{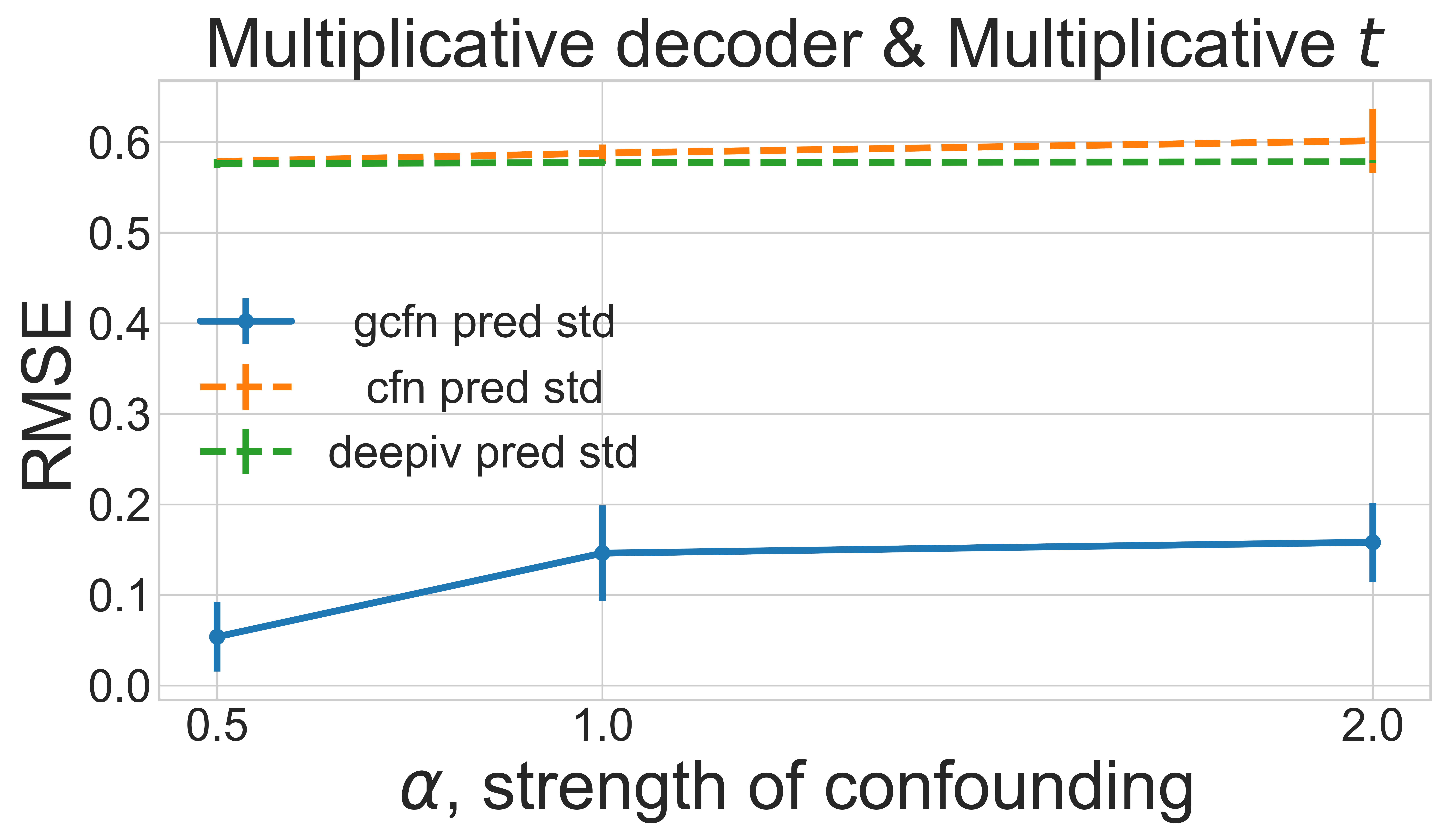}
    \caption{\gls{gcfn} produces better effect estimates than \gls{cfn} and DeepIV when the \textit{additive treatment} process assumption is violated.
    }\label{fig:mul-noise}
  \end{minipage}
  \hspace{0.2cm}
  \begin{minipage}{0.35\textwidth}
    \centering
    \vspace{-13pt}
    \includegraphics[width=1.02\columnwidth]{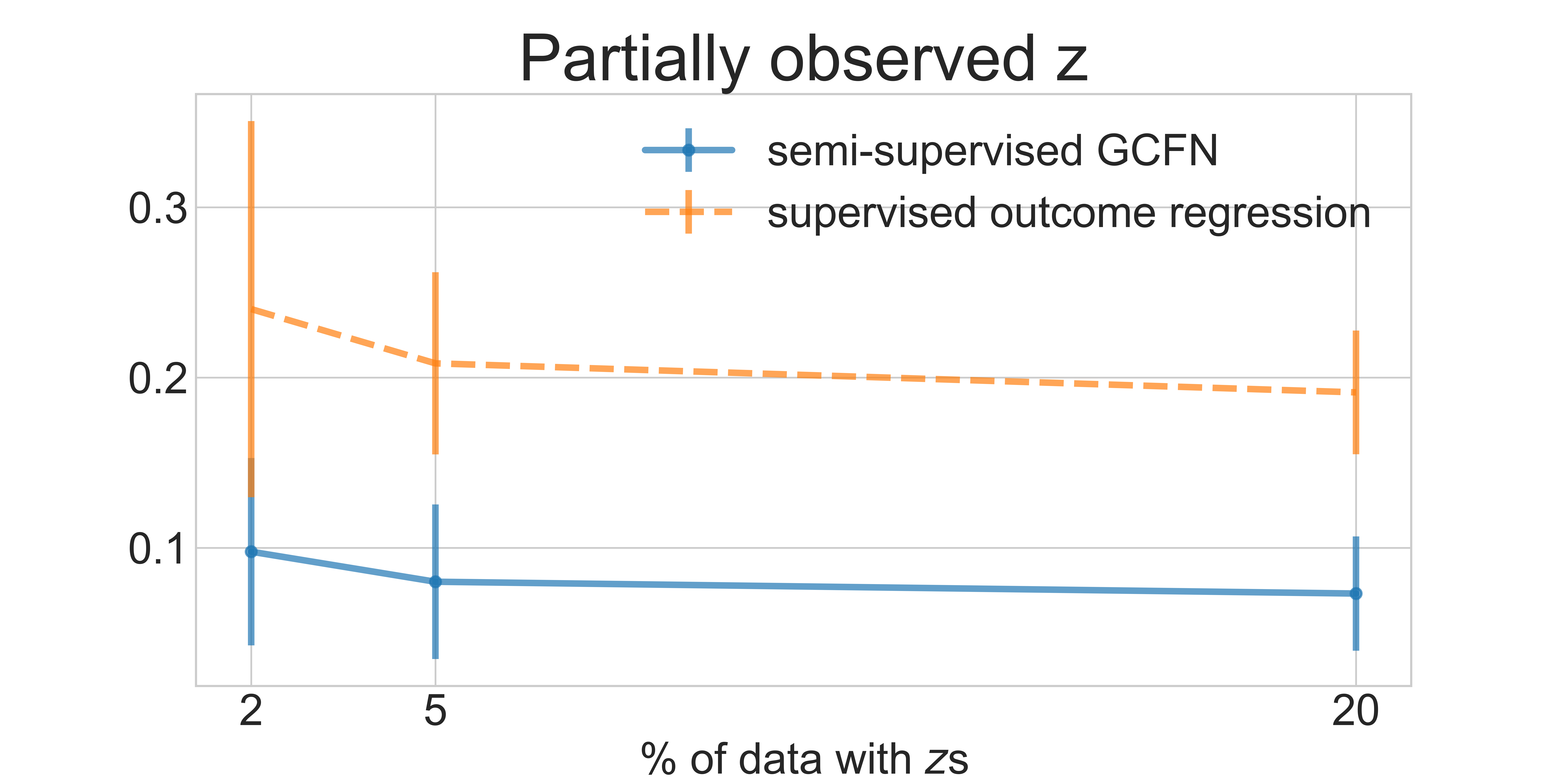}
  \caption{
  Mean RMSE of causal effects of the \gls{gcfn}-predicted causal effects versus percentages of samples with $\mbz$ observed.
  }\label{fig:mixed-exp}
  \end{minipage}
  \vspace{-12pt}
 \end{figure}

\paragraph{Multiplicative outcome \& Additive treatment}
With $\mathcal{N}$ as the normal distribution, we generate
$ \,\, \mbz, \mbepsilon \sim \cN(0,1),\,\, \mbt = (\mbz + \mbepsilon)/\sqrt{2},\,\,
\mby \sim \cN(\mbt + \alpha \mbt^2 \mbz, 0.1),\, $
wher $\alpha$ controls confounding; larger magnitude of $\alpha$ means more confounding.
The generation process above
violates the linear noise relation assumption, $\E[\mbeta_{\mby}|\mbeta_{\mbt}] \propto \mbeta_{\mbt}$, that \gls{cfn} requires \citep{guo2016control}. \gls{gcfn}, on the other hand, does not require this assumption.
In this experiment, \gls{vde} has an additive decoder which specifies a Gaussian reconstruction likelihood: ${\mbt} \sim \cN(h^\prime_{\phi}(\hz) + g^\prime_\phi(\mbepsilon),1)$.
In \Cref{fig:exps}, we compare \gls{gcfn} to \gls{cfn} and DeepIV, and show that \gls{gcfn} produces the best causal effect estimates.
Unlike the others, \gls{gcfn} can adjust for confounding when the outcome process is not additive.
Averaged over all $\alpha$s, \gls{gcfn} outperforms the baselines with an  RMSE of $\boldsymbol{0.09 \pm 0.06}$ compared to \gls{cfn}'s $\boldsymbol{0.58 \pm 0.01}$, \gls{2sls}'s $\boldsymbol{0.55\pm 0.58}$, and DeepIV's $\boldsymbol{0.25\pm 0.17}$.

\paragraph{Multiplicative treatment \&  Additive outcome.}
For this simulation, we generate data as follows:
$\,\, \mbz, \mbepsilon \sim \cN(0,1),\,\, \mbt = \mbz\mbepsilon, \,\, \mby \sim \cN(\mbt + \alpha \mbz, 0.1).$
In this experiment, \gls{vde} has a \textit{multiplicative} decoder which specifies a gaussian reconstruction likelihood with $\mbt = \cN( h^\prime_{\phi}(\hz)g^\prime_\phi(\mbepsilon),1)$.
The \gls{2sls} method uses a linear model $\mbt  = \beta\mbepsilon + \mbeta_{\mbt}$ which will correctly estimate $\E[\mbt\g \mbepsilon]=0$ in our generation process.
\Cref{fig:mul-noise} shows that \gls{gcfn} out-performs \gls{cfn} and DeepIV and is robust to different strengths of confounding ($\alpha \in\{0.5,1,2\}$).
Averaged over all $\alpha$s, \gls{gcfn} outperforms the baselines with an  RMSE of $\boldsymbol{0.13 \pm 0.08}$ compared to \gls{cfn}'s $\boldsymbol{0.58 \pm 0.02}$, \gls{2sls}'s $\boldsymbol{0.55\pm 0.56}$, and DeepIV's $\boldsymbol{0.58\pm 0.01}$.
We omit \gls{2sls} from~\cref{fig:exps} because it performs strictly worse than DeepIV, its deep variant.
DeepIV gives effect-estimates that are close to $0$.
We justify this in~\cref{subsec:deepiv}.

\subsection{\gls{gcfn} with confounders observed on a subset}\label{sec:partial-obs}

In this experiment, we demonstrate that semi-supervised \gls{gcfn} does not need outcome or treatment process assumptions if the confounder $\mbz$ is observed on a subset of the data.
Let $\rho$ be the fraction with $\mbz$ observed and $\cB$ be the Bernoulli distribution.
We generate a mask $\mbm\sim\cB(\rho)$ and data
 $ \,\, 
 \mbepsilon, \mbz\sim\cN(0,1), \quad  \mbt = \mbepsilon \mbz, \quad\, \mby \sim \cN(\mbt + \mbt\mbz,0.1).
 $
Let $\mbz' = \mbz*\mbm$. We observe $(\mby,\mbt,\mbepsilon,\mbz',\mbm)$.
The structurally unrestricted decoder uses a categorical reconstruction likelihood: $p_\phi(\mbt=j\g\hz=z,\mbepsilon=\epsilon) \propto \exp{(g_{\phi}(z,\epsilon,j))}$.
The treatment $\mbt$ is discretized into $50$ bins.
The intervals $[-\infty, -3.5]$ and $[3.5, \infty]$ correspond to one bin each and the interval $[-3.5, 3.5]$ is split into 48 equally-sized bins.
This suffices because few samples fall outside $[-3.5,3.5]$.
For semi-supervised \gls{gcfn}, \gls{vde}'s objective has an additional term defined on the samples with observed $\mbz$'s (\cref{eq:semi-gcfn}). 
The confounder $\mbz$ is split into bins the same way as the treatment.
The additional term for the $i^{th}$ sample is the categorical log-likelihood of the observed $(t_i, \epsilon_i,z_i)$ with respect to the encoder-specified distribution: $q(\hz = z_i \g \mbt=t_i,\mbepsilon=\epsilon_i) \propto \exp(f_\theta (t_i,\epsilon_i,z_i))$. We set
the scaling $\zeta$ on this additional term to be $0.5$.
We report results for $\kappa= 0.1$.
For other $\kappa \in \{0.2, 0.3\}$, results were similar or better.

We compare semi-supervised \gls{gcfn} against regression
with the same outcome model as the baseline, trained only on samples with the confounder observed. We estimated
this ``supervised'' baseline in the same manner as the outcome stage
of \gls{gcfn}.
\Cref{fig:mixed-exp} plots the RMSE of the predicted causal effects vs.\ percentage of samples with observed $\mbz$'s in \cref{fig:mixed-exp}. 
If the data has $2\%$ or more samples with the confounder observed, \gls{gcfn} estimates effects better than the supervised baseline.

\subsection{\gls{gcfn} on high-dimensional Covariates}\label{sec:mnist-exp}
In this experiment, we evaluate \gls{gcfn} on a non-linear simulation given in 
\citet{DBLP:conf/icml/HartfordLLT17} to demonstrate that DeepIV improves upon \gls{2sls}.
Their generation models the effect of price ($\mbt$) on sales ($\mby$), given customer covariates ($\mbx$, MNIST image), and time $s$; they use fuel price as an \gls{iv}.
The outcome is generated using the label of the MNIST image, which denotes customer price sensitivity.
The data generation process for $\mbt$ is additive in \gls{iv} and confounder.
Following this, we use the same additive decoder in \gls{vde} as in~\cref{sec:decoder-structure}, but with time $s$ as an additional input.
We give further experimental details and~\citet{DBLP:conf/icml/HartfordLLT17}'s data generating process in~\cref{appsec:mnist-exp}.

We report effect MSE on a fixed \gls{oos}.
We compare against \citet{DBLP:conf/icml/HartfordLLT17}'s reported results for two sample sizes, $10,000$ and $20,000$.
DeepIV's reported results exclude a few large effect MSE outliers; we do not exclude such errors for \gls{gcfn}.
We report \gls{gcfn}'s performance over $10$ seeds.
Overall, \gls{gcfn} performed on par or better than DeepIV.
First, we report \gls{gcfn}'s effect MSE with $\kappa=0.2$.
For $10,000$ samples, \gls{gcfn} produced effect MSEs that ranged in $[\boldsymbol{0.30- 0.42}]$, better  than DeepIV's reported range of around $[\boldsymbol{0.30-0.50}]$ (which is almost twice as large).
For $20,000$ samples, \gls{gcfn}'s effect MSE range improved to $[\boldsymbol{0.25-0.40}]$ while DeepIV reported a performance of around $[\boldsymbol{0.25-0.45}]$.
For both sample sizes, we note that $\kappa=0.1,0.3$ gave similar results.
To see this, for $20,000$ samples, averaged over $10$ seeds, \gls{gcfn} achieved a mean effect MSE of $\boldsymbol{0.305}$ or better for any $\kappa\in \{0.1, 0.2, 0.3\}$, beating DeepIV's $\boldsymbol{0.32}$.

\subsection{\gls{gcfn} on high-dimensional \glspl{iv}}\label{sec:mnist-iv-exp}

In this experiment, we evaluate \gls{gcfn} on data with a high-dimensional \gls{iv}.
\citet{bennett2019deep} use the following data generating process to demonstrate DeepGMM~\citep{bennett2019deep} improves upon existing methods:
  $ \,\, \mbeps \sim \cU [-3,3] \quad \mbz \sim \cN(0,1) \quad \mbt \sim \cN(\mbz + \mbeps, 0.1 ) \quad \mby = \cN\left(|\mbt| + \mbz,0.1\right).$ 
However, the scalar $\mbeps$ is not directly observed.
Instead, $\mbeps$ is mapped to a digit $\{0, \ldots, 9\}$ and a corresponding MNIST image $\mbeps_M$ is given as the \gls{iv}.
To estimate effects well with such an \gls{iv}, any method must learn to label the MNIST image.
In this setting, \gls{vde}'s encoder and decoder both take an embedding $\ell_\gamma(\mbeps_M)\in \mathbb{R}^{10}$ as input.
The embedding $\ell_\gamma$ is trained in \gls{vde} along with the encoder and decoder.
Respecting the additive treatment process, we specify an additive decoder.

We ran \gls{gcfn} with $10$ different random seeds and report results for $\kappa=0.3$, chosen based on mean test outcome MSE
($0.136 \pm 0.008)$.
\gls{gcfn} performs competitively with an effect MSE of $\boldsymbol{0.077 \pm 0.022}$ compared to DeepGMM's $\boldsymbol{0.07 \pm 0.02}$ and DeepIV's $\boldsymbol{0.11 \pm 0.00}$, both as reported in~\cite{bennett2019deep}.
Effect MSE for $\kappa \in\{0.2, 0.4\}$ were similar and within standard error of DeepGMM's performance. See~\cref{appsec:mnist-iv-exp} for further experimental details and results.

\subsection{The Effect of Slave Export on Trust}
We demonstrate the recovery of the causal effect of slave export on the trust in the community~\citep{nunn2011slave}.
\citet{nunn2011slave} pooled surveys and historical records to get sub-ethnicity and tribe level data from the period of slave trade.
The data was used to study the long-term effects of slave-trade, measured in the 2005 Afrobarometer survey.
We predict the effect of the treatment $\mbt=$\textbf{ln(1 + slave-export/area)} on the outcome of interest, $\mby=$\textbf{trust in neighbors}.
The dataset has $6932$ samples with $59$ features.
After filtering out missing values, we preprocessed $46$ covariates and \gls{iv} to have mean $0$ and maximum $1$, and the treatment $\mbt$ to lie in $[0,2]$.
The authors claim that the distance to sea cannot causally affect how individuals trust each other, but it affects the chance of coming in contact with colonial slave-traders and being shipped to the Americas, making it an \gls{iv}.
They control for urbanization, fixed effects for sophistication, political hierarchies beyond community, integration with the rail network, contact with European explorers, and missions during colonial rule.

For this experiment, \gls{vde}'s decoder $g_\phi$ specifies a categorical reconstruction likelihood as $p_\phi(\mbt=i\g\hz=z,\mbepsilon=\epsilon) \propto \exp{(g_{\phi}(z,\epsilon,i))}$.
Each category of the treatment corresponds to one of 50 equally-sized bins in the interval $[0,2]$.
\citet{nunn2011slave} use a linear model for the outcome $\mby$ and use the distance to sea as an \gls{iv} for each community.
We also use a partially linear model $\mby=\beta \mbt + h_\theta(\hz)$ so that the effect we recover is of comparable nature to the effect reported in the paper.
The outcome network $h_\theta$ has 2 layers with 10 hidden units each and ReLUs.

Averaged over $4$ mutual information coefficients $\kappa$ and $5$ random seeds, \gls{gcfn}'s estimate of $\beta$ was $\boldsymbol{-0.21 \pm 0.04}$ compared with $\boldsymbol{-0.27\pm 0.10}$, as reported by \citet{nunn2011slave}.

%% file: sections/future.tex
\section{Discussion and Future}\label{sec:future}
\vspace{-3pt}
In this paper, we characterize general control functions for causal estimation.
General control functions allow for effect estimation without structural outcome process assumptions like \gls{2sls} or \gls{cfn}.
The key challenge in building general control functions is ensuring joint independence between the \gls{iv} and the control function and (unobserved) true confounder.
Joint independence can be guaranteed via structural treatment process assumptions, like additivity or monotonicity.
We develop the \glsreset{gcfn}\gls{gcfn} to build general control functions and estimate effects with them.
Further, we develop semi-supervised \gls{gcfn} which uses confounders observed on a subset of the data to construct general control functions without treatment process assumptions.
Finally, we consider imperfect estimation of the general control function and bound average error in effects using quantities optimized in \gls{vde}.

\vspace{-3pt}
\paragraph{Tradeoffs with assumptions.}
In causal estimation, parametric assumptions can be traded-off with assumptions of strength of \gls{iv} or positivity.
Consider a setting where $\mbepsilon$ is binary.
For every possible confounder value, only two values of the treatment are observed.
Thus it is impossible to estimate a quadratic function of $\mbt$ for each fixed value of the confounder.
This means $\mby\g \mbt$ is not identified without strong assumptions like linearity in $\mbt$. Incorporating outcome properties, like the conditional independence 
$\mby\indep \mbepsilon\g \mbt,\mbz$, 
into control function estimation would
be a fruitful direction.

%% file: sections/app-theory.tex
\section{Theoretical Details and Proofs}\label{appsec:theory}
\paragraph{Notation} We use the expectation operator in different contexts in the proof. $\E_q$ denotes expectation with respect to the density $q$ and $\E_{\mbz}$ denotes expectation with respect to the density of the random variable $\mbz$. When the density function or the random variable are clear from the context, we drop the subscript and use $\E$.
\subsection{The general IV causal graph with covariates/observed confounders}
\begin{figure}[!htbp]
      \begin{center}
        \begin{tikzpicture}
        \node[state,fill=gray] (z) at (2.5,1) {$\mbz$};
        \node[state] (x) at (1,1) {$\mbx$};
        \node[state] (e) at (0,0) {$\mbepsilon$};
        \node[state] (t) [right =of e] {$\mbt$};

        \node[state] (y) [right =of t] {$\mby$};

        \path (e) edge (t);
        \path (x) edge (t);
        \path[dashed] (z) edge (t);
        \path (t) edge (y);
        \path (x) edge (y);
        \path (x) edge (e);
        \path[dashed] (z) edge (y);
        \path[dashed] (z) edge (t);
        \path[bidirected] (x) edge[bend left=60] (z);
    \end{tikzpicture}
      \end{center}
      \caption{Causal graph with hidden confounder $\mbz$, outcome $\mby$, \gls{iv} $\mbepsilon$, treatment $\mbt$ and covariates $\mbx$.}\label{fig:full-iv-graph}
      \vspace{-5pt}
\end{figure}
\Cref{fig:full-iv-graph} is the general version of the \gls{iv} problem where the instrumental variable property holds true after conditioning on $\mbx$. This is sometimes called a conditional instrument. All our proofs and results carry over to the situation with covariates after conditioning all estimables and distributions on $\mbx$. \gls{vde} in this setting with covariates is re-written as:
\begin{align}
\max_{\theta,\phi}&\,  \E_{F(\mbt,\mbepsilon,\mbx)}\E_{q_{\theta }(\hz\g \mbt,\mbepsilon,\mbx)}\log p_{\phi }(\mbt\g \hz, \mbepsilon,\mbx)  - \lambda\mathbf{I}_\theta(\hz; \mbepsilon | \mbx)
\end{align}
\subsection{Mutual Information lower bound}\label{appsec:milb}
\def\qt{q_{\theta}}
\def\pp{p_{\phi}}
Here, we show the full derivation of the lower bound for negative mutual-information.
We derive the lower bound for the general case where there are both observed and unobserved confounders.
A simple lower bound can be obtained by using $\Hb(\hz\g \mbepsilon,\mbx)\geq \Hb(\hz\g \mbepsilon,\mbt,\mbx )$, but this cannot be made tight unless $\mbepsilon$ completely determines $\mbt$.
Therefore, we cannot guarantee independence unless the data at hand is not confounded. 
Instead we introduce two auxiliary distributions $r_\nu(\hz\g \mbx)$ and $\pp(\mbt\g \mbepsilon, \hz, \mbx)$, following the work in variational inference
~\citep{ranganath2016hierarchical,agakov2004auxiliary,salimans2014markov,maaloe2016auxiliary} and causal inference~\citep{ranganath2018multiple}.

We let $F(\mbt,\mbx,\mbepsilon,\mby)$ be the true data distribution and $\qt(\hz\g \mbt,\mbepsilon, \mbx=x)$ be the control function distribution.
We overload notation and also use $\qt$ to refer to any distribution that involves operations with $\qt(\hz\g \mbt,\mbepsilon, \mbx=x)$. We use $\myeq$ to denote that the LHS and RHS are equal up to constants that are ignored during optimization.
In the following, both $\Hb(\mbt,\mbepsilon\mid \mbx=x)$, $\Hb(\mbepsilon|\mbx=x)$ are constants with respect to the parameters of interest $\phi, \theta,\nu $ and we will drop them from the lower bound when encountered.
For a given $\mbx=x$, we lower-bound the negative instantaneous conditional mutual information:
\allowdisplaybreaks{}
\begin{align*}
-\lambda&\mathbf{I}(\hz;\mbepsilon \g  x )   = -\lambda \KL{\qt(\hz,\mbepsilon\g x)}{\qt(\hz\g x)F(\mbepsilon\g x)}
\\
    = & -\lambda \left[\E_{\qt(\mbepsilon,\hz\g x)}\left[\log \qt(\mbepsilon \g  \hz,x) - \log F(\mbepsilon\g x)\right]\right]
                                    \\
    = & -\lambda \left[\E_{\qt(\mbepsilon,\hz\g x)}\left[\log \qt(\mbepsilon \g  \hz,x)\right] + \Hb(\mbepsilon\g x)\right]
                                    \\
    \myeq & -\lambda \big[\E_{\qt(\mbepsilon,\hz\g x)} \big[ \KL{\qt(\hz\g x)}{\qt(\hz\g x)}
                                    +\KL{\qt(\mbt\g \mbepsilon,\hz,x)}{\qt(\mbt\g \mbepsilon,\hz,x)} +  \log \qt(\mbepsilon \g  \hz,x)\big]\big]
                                    \\
    \geq & -\lambda \big[\E_{\qt(\mbepsilon,\hz\g x)} \big[\KL{\qt(\hz\g x)}{r_{\nu}(\hz\g x)}  + \KL{\qt(\mbt\g \mbepsilon,\hz,x)}{\pp(\mbt\g \mbepsilon,\hz,x)} + \log \qt(\mbepsilon \g  \hz, x ) \big]\big]
                                    \\
    = &-\lambda \big[\E_{\qt(\mbepsilon,\hz\g x)}\big[ \log \left[\qt (\hz,\mbepsilon \g x) \right]  + \E_{\qt(\mbt\g \mbepsilon,\hz,x)}\log \qt (\mbt \g  \mbepsilon,\hz,x) 
                                    \\ & \quad 
                                    \quad -\E_{\qt(\hz\g x)}\log r_{\nu}(\hz \g x)
                                     - \E_{\qt(\mbt,\mbepsilon,\hz\g x)}\log p_{\phi}(\mbt \g  \mbepsilon,\hz, x)\big]                                     \\
    = & -\lambda \big[\E_{\qt(\mbepsilon,\hz,\mbt\g x)} \log \left[\qt(\hz,\mbepsilon, \mbt \g x)\right] -\E_{\qt(\hz\g x)}\log r_{\nu}(\hz\g x) - \E_{\qt(\mbt,\mbepsilon,\hz\g x)}\log p_{\phi}(\mbt \g  \mbepsilon,\hz,x)\big]
                                    \\
    = & -\lambda \big[\E_{F(\mbt,\mbepsilon\g x)}\E_{\qt(\hz\g \mbepsilon,\mbt,x)} \log \left[\qt(\hz\g \mbt,\mbepsilon,x) - \log \pp(\mbt \g  \mbepsilon,\hz,x)\right] - \Hb(\mbt,\mbepsilon \g x)
                                   \\
                                   & \quad \quad  -\E_{\qt(\hz\g x)}\log r_{\nu}(\hz\g x) \big]
                                    \\
    \myeq & -\lambda\E_{F(\mbt,\mbepsilon\g x)}\left[\KL{q_\theta(\hz\g \mbt,\mbepsilon,x)}{r_{\nu}(\hz\g x)} - \E_{\qt(\hz\g \mbepsilon,\mbt,x)}\log \pp (\mbt \g  \mbepsilon,\hz,x) \right],
\end{align*}

where the hidden term $ -\lambda \left[\Hb(\mbepsilon\g \mbx=x) -  \Hb(\mbt,\mbepsilon\g \mbx=x)\right]$ is a constant for a given instance of the problem. We do not need access to the distribution $\mbt,\hz,\mbepsilon \g  \mbx=x$ because the information that we lower bounded, $\mathbf{I}(\hz;\mbepsilon\g \mbx=x)$, is averaged over $\mbx=x$ in our objective.
Recall that $\pp(\mbt\g \mbepsilon,\hz,\mbx=x)$ is the reconstruction term in \gls{vde}.
This lower bound is tight when the introduced KL terms are 0, which occurs when $r_\nu(\hz \g  \mbx=x)= \qt(\hz\g \mbx=x)$ and $p_{\phi}(\mbt\g \mbepsilon,\mbx=x,\hz)= \qt (\mbt\g \mbepsilon,\mbx=x,\hz)$.
This means that if the models $\pp,r_\nu$ are rich enough, the gap between the lower bound and mutual information can be optimized to be zero.
The second term $\E_{\qt(\hz\g \mbepsilon,\mbt,\mbx=x)}\log \pp (\mbt \g  \mbepsilon,\hz,\mbx=x) $ is the same as the reconstruction likelihood.
Thus substituting the lower bound into the full objective with given covariates gives
\begin{align*}
\E_{F(\mbt,\mbepsilon,\mbx)}\left[ (1 + \lambda)\E_{\qt(\hz\g \mbt, \mbepsilon,\mbx)}\log \pp (\mbt \g  \mbepsilon,\hz,\mbx)  -\lambda\KL{q_\theta(\hz\g \mbt,\mbepsilon,\mbx)}{r_{\nu}(\hz\g \mbx)} \right]
\end{align*}

\paragraph{Optimization for \glsreset{vde}\gls{vde}.}
The \gls{vde} optimization involves the expectations of distributions with parameters with respect to a distribution that also has parameters.
For distributions that are not being integrated against, we can move the gradient inside the expectation.
For distributions that are integrated against, score-function methods provide a general tool to compute stochastic gradients;~\citet{glasserman2013monte,williams1992simple,ranganath2014black,mnih2014neural}.
In our experiments, we let the control function be a categorical variable.
This allows us to marginalize out the control function and compute the gradient.

\subsection{Proof of Theorem 1}
\setcounter{thm}{0}
\setcounter{prop}{0}
\begin{thm} \textbf{(Meta-identification result for general control functions)}\\
\theoremone{}
\end{thm}

We prove this for the setting without covariates. The proof adapts to the setting with covariates (observed confounders) by conditioning all terms on them.

\begin{proof} (\Cref{thm:intro})
The proof shows that reconstruction (A1) and joint independence (A2) together imply ignorability, and strong \gls{iv} (A3) together with the joint independence (A2) imply positivity.

\paragraph{Ignorability.}

To establish ignorability we need to show that $\mby_{t} \indep \mbt\g \hz$ where $\mby_{t}$ is the potential outcome for a unit when the treatment given is $\mbt=t$.
The outcome $\mby$ is constructed from the potential outcomes by indexing the one $\mby_{t^*}$ corresponding to the observed treatment $\mbt=t^*$.

By assumption A2, we have the joint independence $\mbepsilon\indep (\mbz,\hz)$ which implies
 \[\mbepsilon \indep (\mbz,\hz)\,\ \implies \mbepsilon \indep \mbz \g \hz = \zhat\quad \forall \zhat \in \supp(\hz). \]
Note that by the reconstruction property (from assumption A1) $\mbt = d(\hz,\mbepsilon)$.
So given $\hz$, $\mbt$ is purely a function of $\mbepsilon$.
Thus, given $\hz$, $\mbt$ satisfies the same conditional independence as $\mbepsilon$: $\mbepsilon \indep \mbz\g \hz$.
Using this, we have
\[\mbepsilon \indep \mbz\g \hz \implies d(\hz,\mbepsilon) \indep \mbz \g \hz \implies \mbt \indep \mbz\g \hz. \]

The potential outcome $\mby_t$ depends only on $\mbz$ and some noise $\mbeta$ that is jointly independent of all other variables. This means for some function $m_t$ such that $\mby_t = m_t(\mbz,\mbeta)$.
\[ \mbt \indep \mbz\ \g \hz \implies \mbt\indep  m_t(\mbz,\mbeta) \g \hz \implies \mbt \indep \mby_t \g \hz. \]
This shows ignorability.

\paragraph{Strength of \gls{iv} and Positivity.}

Positivity means that for almost every $t\in \supp(\mbt)$,
\[q(\hz)>0, \implies q(\mbt=t\g\hz)>0.\]
We start with $q(\mbt\g\hz)$ and expand  it as an integral over the full joint.
\begin{align}\label{eq:positivity}
\begin{split}
    q(\mbt \g \hz) &
                = \int q(\mbt\g \mbz=z,\hz,\mbepsilon=\epsilon,\mbdelta=\delta, \mbt ) q(\mbepsilon=\epsilon\g \mbz=z,\hz,\mbdelta=\delta) q(\mbz=z,\mbdelta=\delta \g \hz)dz d\delta d\epsilon  \\
                & = \int q(\mbt\g \mbz=z, \mbepsilon=\epsilon,\mbdelta=\delta) q(\mbepsilon=\epsilon\g \mbz=z,\hz,\mbdelta=\delta) q(\mbz=z,\mbdelta=\delta \g \hz)dz d\delta  d\epsilon
                \\
                & \quad \quad \{\text{by } \mbt = g(\mbz,\mbepsilon,\mbdelta)\}
                \\
                & = \int q(\mbt \g \mbz=z, \mbepsilon=\epsilon,\mbdelta=\delta) q(\mbepsilon=\epsilon \g \mbz=z, \mbdelta=\delta) q(\mbz=z,\mbdelta=\delta \g \hz)dz d\delta  d\epsilon
                \\
                & \quad \quad \{\text{by A2: } \mbepsilon \indep (\mbz,\hz,\mbdelta)\}
                \\
                & = \int \left[\int q(\mbt \g \mbz=z, \mbepsilon=\epsilon,\mbdelta=\delta) q(\mbepsilon=\epsilon \g \mbz=z, \mbdelta=\delta) d\epsilon \right]q(\mbz=z,\mbdelta=\delta \g \hz)dz d\delta
                \\
                & = \int F'(\mbt \g \mbz=z, \mbdelta=\delta)q(\mbz=z,\mbdelta=\delta \g \hz)dz d\delta
\end{split}
\end{align}
Note that $q(\mbz=z,\mbdelta=\delta\g \hz)$ is a valid density over $(\mbz=z,\mbdelta=\delta)$
\footnote{If $q(\mbz=z,\mbdelta=\delta\g \hz)=0$ everywhere then no pair $(\mbz=z,\mbdelta=\delta)$ maps to $\hz$ and $\hz$ cannot be observed and we cannot condition on it. But $\hz$ is constructed explicitly as part of the algorithm, so it's observed.
Thus $q(\mbz=z,\mbdelta=\delta\g \hz)$ is a valid conditional density.}.
Under assumption A3, for any compact set $B\subseteq \supp(\mbt)$ and for almost every $ t\in B$, 
\begin{align}\label{eq:positivity-last}
  \begin{split}
    q(\mbt =t \g \hz) & = \int F'(\mbt=t \g \mbz=z, \mbdelta=\delta)q(\mbz=z,\mbdelta=\delta \g \hz)dz d\delta
                        \\
                        & \geq c_B  \int q(\mbz=z,\mbdelta=\delta \g \hz)dz d\delta 
                        \\
                        & = c_B > 0
  \end{split}
\end{align}
However, almost every $t\in \supp(\mbt)$ is contained in some compact subset $B\subseteq \supp(\mbt)$.
Thus, \cref{eq:positivity-last} holds for almost every $t \in \supp(\mbt)$, meaning that positivity is satisfied.
\def\do{\text{do}}

\paragraph{Computing the causal effect.}
Given ignorability and positivity, the true causal effect (a.e.\ in $\supp(\mbt)$) is determined as a property of the distribution $q(\hz, \mbt, \mby)$ as follows:
\[\E_{q(\hz)}\E[\mby \g \hz,\mbt=t] = \E_{q(\hz)}\E[\mby \g \hz,\do(\mbt=t)] = \E[\mby\g \do(\mbt=t)]\]

\end{proof}

\paragraph{Assumptions for continuous $\mbt$.}
When $\mbt$ has non-zero density rather than non-zero probability given the general control function, the true expected outcome being continuous everywhere as a function of the treatment is a sufficient condition for the causal effect estimation for almost all treatment values.

\subsection{Marginal Independence does not imply joint independence}\label{app:marginal-v-joint}
Here, we build an example of a function
of two independent variables $\mba, \mbb$ that is marginally independent of both.
Let $1_e{}$ be one if $e$ is true and zero if not,
\begin{align*}
    & \mba,\mbb \sim \textrm{uniform}(0,1),\\
    & \mbc(\mba,\mbb) = 1_{\mba + \mbb > 1} (\mba + \mbb - 1) +  1_{\mba + \mbb \leq 1} (\mba + \mbb).
\end{align*}
First, $\mbc$ is marginally a  uniform variable.\footnote{
    $P(\mbc<x) = P(\mba+\mbb<x) + P(1< \mba + \mbb <1 + x)
             = 0.5(x^2 - 1) + 1 - 0.5(1-x)^2 = x$.}
The distribution $\mbc\g \mba=x$
can be obtained by translating the distribution of $\mbb$ up by $x$, then
translating the part greater than one down to zero, meaning $\mbc\g \mba$ is uniformly distributed.
Thus $p(\mbc\g \mba) = p(\mbc)$ meaning $\mbc\indep \mba$. However, $\mbc$ is a deterministic
function of $\mba$ and $\mbb$. Therefore, while $\mbc\g \mba$ is
uniformly distributed, $\mbc\g (\mba, \mbb)$ is a dirac-delta distribution,
meaning $p(\mbc\g \mba, \mbb) \not = p(\mbc \g \mba)$ implying $\mbc \nindep \mba \g \mbb$.
Note that $\mbb$ can be constructed back from $\mbc,\mba$ up to measure-zero as
$\mbb = \mbc - \mba$  if $\mbc>\mba$ and $\mbb = \mbc - \mba + 1$ if $\mbc \leq \mba$; i.e., $\mbc$ is almost everywhere invertible for each fixed $\mba=a$.

This construction with uniform random variables can be generalized
to other continuous distributions by inverse transform sampling.
Any marginal density of $\mba,\mbb$ can be bijectively mapped to a uniform density over $[0,1]$.
Then $\mbc$ can be computed as above and then $\mba,\mbb,\mbc$ can be bijectively mapped back; $\mbc$ could be mapped back with the CDF of $\mbb$.
Conditional dependence is unaffected by bijective transformations and therefore the issue remains.
Similar constructions exist with discrete random variables. In general, assumptions on the true data generating process will be needed to ensure joint independence.

\subsection{From additive treatment processes to joint independence}\label{appsec:add-dec}\label{appsec:add-dec}

Consider treatment processes of the form $\mbt = h(\mbz, \mbdelta) + g(\mbeps)$.
Let the reconstruction map be additive:
\begin{align*}
\mbt = h^\prime(\hz) + g^\prime(\mbeps).
\end{align*}
Consider the random variable $\mbt - \E[\mbt\g \mbeps]$ which is sampled as follows: $\mbeps\sim h(\mbeps), \mbz\sim h(\mbz), \mbdelta\sim h(\mbdelta)$ and $\mbt - \E[\mbt\g \mbeps]=h(\mbz, \mbdelta) + g(\mbeps) - \E_{\mbz, \mbdelta}[h(\mbz, \mbdelta) + g(\mbeps)]$.
We show that  $h^\prime(\hz)$ determines $h(\mbz,\mbdelta)$ by expressing the random variable $\mbt - \E[{\mbt\g \mbepsilon}]$ in terms of $\mbz,\mbdelta$ and $\hz$
\[h^\prime(\hz) - \E_{\hz}[h^\prime(\hz)] = \mbt - \E[\mbt\g \mbepsilon] = h(\mbz,\mbdelta) - \E_{\mbz,\mbdelta}[h(\mbz,\mbdelta)].\]
Therefore for some constant $c$, $h^\prime(\hz) = h(\mbz,\mbdelta) + c$.
By the independence, $\hz\indep \mbepsilon$, we have
\[ q(\hz,h(\mbz,\mbdelta) \g \mbepsilon) =  q(\hz , h^\prime(\hz) - c \g \mbepsilon) = q(\hz, h^\prime(\hz) - c) = q(\hz,h(\mbz,\mbdelta)).\]
Thus we have $(\hz,h(\mbz,\mbdelta) )\indep \mbepsilon$.
See~\cref{lem:general-t-process-joint-ind} for the proof that $(\hz,h(\mbz,\mbdelta) )\indep \mbepsilon$ implies the joint independence $\mbeps \indep (\hz, \mbz, \mbdelta)$ for any treatment process $\mbt = g(\mbeps, h(\mbz, \mbdelta))$, including $\mbt=g(\mbeps) + h(\mbz, \mbdelta)$.

\def\eps{\mbepsilon}
\subsection{Joint independence treatments of the form $\mbt = g(\mbeps, h(\mbz, \mbdelta))$}\label{appsec:general-t-process}

General control functions for treatments of the form $\mbt = g(\mbeps, h(\mbz, \mbdelta))$, unlike $\mbt = g(\mbeps, \mbz)$, require a stronger joint independence $\mbeps \indep (\mbz, \hz, \mbdelta)$ to guarantee ignorability (A2, \cref{thm:intro}).
The structural assumptions --- that helped guarantee $\mbeps \indep (\mbz, \hz)$ above --- can guarantee $\mbeps \indep (h(\mbz, \mbdelta), \hz)$.
Here, we show that $\mbeps \indep (h(\mbz, \mbdelta), \hz) \implies \mbeps \indep (\mbz, \hz, \mbdelta)$  in such settings.

\begin{lemma}\label{lem:general-t-process-joint-ind}
Consider treatment process $\mbt = g(\mbeps, h(\mbz, \mbdelta))$ and the joint independence $(\hz,h(\mbz,\mbdelta) )\indep \mbepsilon$ holds.
Then, if $\hz = e(\mbt, \mbeps)$, the joint independence $(\hz,\mbz,\mbdelta)\indep \mbepsilon$ holds.
\end{lemma}
\begin{proof} We begin by showing $ q(\hz\g \mbz,\mbepsilon,\mbdelta) =  q(\hz\g h(\mbz,\mbdelta))$:
\begin{align}\label{eq:step1}
\begin{split}
    q(\hz\g \mbz,\mbepsilon,\mbdelta) & = \int q(\hz \g \mbz,\eps,\mbt=t,\mbdelta)q(\mbt=t \g \eps,\mbz,\mbdelta )dt \quad \{\text{full joint expansion}\}
    \\
                        & = \int q(\hz \g \eps,\mbt=t)q(\mbt=t \g \eps,\mbz,\mbdelta )dt \quad \{\hz\indep (\mbz,\mbdelta) \g \eps, \mbt=t\}
                        \\
                        & = \int q(\hz \g \eps,\mbt=t)q(\mbt=t \g \eps,h(\mbz,\mbdelta) )dt \quad \{\mbt = g(\mbeps, h(\mbz, \mbdelta))\}
                        \\
                        & = \int q(\hz \g \eps,\mbt=t,h(\mbz,\mbdelta))q(\mbt=t \g \eps,h(\mbz,\mbdelta) )dt \quad \{\hz\indep h(\mbz,\mbdelta) \g \eps, \mbt=t\}
                        \\
                        & = q(\hz \g \mbepsilon, h(\mbz,\mbdelta))
                        \\
                        & = q(\hz \g h(\mbz,\mbdelta)) \quad  \{(\hz,h(\mbz,\mbdelta)) \indep \eps\}
    \end{split}
\end{align}
Integrating both sides with respect to $q(\mbepsilon \g \mbz,\mbdelta) $ we get
\begin{align}\label{eq:step2}
  \begin{split}
    \int q(\hz \g h(\mbz,\mbdelta)) q(\eps=\epsilon\g \mbz,\mbdelta)d\epsilon 
    & = \int q(\hz\g \mbz,\mbepsilon=\epsilon,\mbdelta) q(\eps=\epsilon \g \mbz,\mbdelta) d\epsilon =  q(\hz\g \mbz,\mbdelta)
  \end{split}
\end{align}
Now, the LHS in~\cref{eq:step2} is
\[\int q(\hz \g h(\mbz,\mbdelta)) q(\eps=\epsilon\g \mbz,\mbdelta)d\epsilon  = q(\hz \g h(\mbz,\mbdelta))\implies q(\hz \g h(\mbz,\mbdelta)) = q(\hz\g \mbz,\mbdelta).\]
This means
\[q(\hz\g \mbz,\mbepsilon,\mbdelta) = q(\hz \g h(\mbz,\mbdelta)) = q(\hz\g \mbz,\mbdelta)\]
Thus $ (\hz, h(\mbz,\mbdelta))\indep \mbepsilon$ implies the joint independence $(\hz, \mbz,\mbdelta)\indep \mbepsilon$.
\end{proof}

\paragraph{Note.} The proof above shows that we can recover a control function that satisfies ignorability.
In this additive setting with finite support however, both the control function and the true confounder violate another fundamental assumption in causal estimation: \emph{positivity}.
To see this violation of positivity notice that $p(\mbt > a+ \max_{\epsilon\in \supp(\mbeps)} g(\epsilon) \g h(\mbz, \mbdelta)= a) = 0$ for any $a$ such that $p(\mbt > a + \max_{\epsilon\in \supp(\mbeps)} g(\epsilon) ) > 0$ and $p(h(\mbz, \mbdelta)=a)>0$.
When positivity is violated, further assumptions are needed to compute causal effects on the whole support of $\mbt$ in general.
Without further assumptions, effects can only be computed on a compact subset of $B \subseteq \supp(\mbt)$ within which positivity holds.

\subsection{From monotonic treatment processes to joint independence}\label{appsubsec:monotonicity}
\citet{imbens2009identification} explored identification for settings where the outcome process is non-separable but the treatment is a strictly monotonic function of the unobserved confounder.
We show that if the reconstruction map $d(\hz, \mbeps)$ reflects this monotonicity condition and $\hz \indep \mbeps$, the control function is determined by the true confounder and therefore joint independence holds.
In \gls{vde}, the decoder would be monotonic to reflect this assumption.
\begin{lemma}
  Let $\mbepsilon$ and $\mbz$ be the true \gls{iv} and confounder respectively. Let $\mbz$ be a continuous scalar.
  \vspace{-5pt}
  \begin{enumerate}[leftmargin=10pt,itemsep=-3pt]
    \item Assume that $\mbz$ has a continuous strictly monotonic \gls{cdf}.
    Let the true treatment process be $\mbt=g(\mbepsilon, \mbz)$ where $g$ is strictly monotonic in the second argument.
    \item Let the control function be $\hz = e(\mbepsilon, \mbt)$ and let $\hz \indep \eps$. Let reconstruction map be $d$ where $\mbt = d(\mbepsilon, \hz)$. Let $e(\cdot, \cdot)$ and $d(\cdot, \cdot)$ be strictly monotonic in the second argument\footnote{Note that $e(\epsilon, \cdot) = d^{-1}(\epsilon, \cdot)$. Then, monotonicity of $d$ in the second argument implies the same for $e$.}.
    \item  Assume that the functions $g,e,d$ are continuous in the second argument and exist for almost every value in the first argument.
  \end{enumerate}
  Then, the control function $\hz$ can be expressed as a deterministic function of the true confounder $\mbz$.
\end{lemma}
\begin{proof}
  \newcommand{\Hhat}{\hat{H}}
First, note that $\mbt$ can be written as a function of $\eps$ and a uniform random variable $\mbu$ using the \gls{cdf}-inverse trick. Let $H(z) = F(\mbz \leq z)$.
By strict monotonicity and continuity of $H$, $H^{-1}$ exists and $\mbz = H^{-1}(\mbu)$ for a uniform random variable $\mbu\indep \mbepsilon$: 
\[\mbt = g(\mbepsilon, \mbz) = g(\eps, H^{-1}(\mbu)) = \ghat(\eps, \mbu).\]
Note that $H^{-1}$ is strictly monotonic.
So, $\ghat$ is a strictly monotonic function in the second argument.

Second, due to $\mbzhat\indep \eps$, the conditional \gls{cdf} of $\mbzhat \g \mbeps=\epsilon$ is the same as the marginal \gls{cdf} as $\mbzhat$ for almost every value $\epsilon\in \supp(\mbeps)$; let's call this \gls{cdf} $\Hhat$.
By the definition $\hz = e(\mbeps, \mbt)$ we can express $\mbzhat = e(\eps, \ghat(\eps, \mbu))$.
Now, $e(\cdot, \cdot), \ghat(\cdot, \cdot)$ are both continuous and strictly monotonic in the second argument.
So, $\mbzhat$'s \gls{cdf} $\Hhat$ is also strictly monotonic and $\Hhat^{-1}$ exists and is again strictly monotonic.
Therefore, for almost any $\epsilon \in \supp(\mbeps)$, we can construct a new uniform random variable by applying $\mbzhat$'s \gls{cdf} $\Hhat$ to $\mbzhat$:
\[\mbv = \Hhat(\mbzhat) = \Hhat(e(\eps, \ghat(\eps, \mbu))).\]
For simplicity, let $\mbv = J(\mbepsilon, \mbu)$. Note $J(\cdot, u)$ is strictly monotonic in $u$ by strict monotonicity of $\Hhat, \ghat$ in their second arguments.
So, we can write $\mbu$'s \gls{cdf} in terms of $\mbv$'s \gls{cdf}:
\[a = P(\mbu < a) = P(\mbv < J(\epsilon, a)) = J(\epsilon, a).\]
This means that $J(\epsilon,a)$ is an identity function for almost any $\epsilon\in \supp(\mbeps)$.

Finally, we can write $\mbzhat$ as a function of $\mbz$ for almost any $\epsilon\in \supp(\mbeps)$, completing the proof:
\begin{align*}
  \mbzhat & = \Hhat^{-1}(J(\epsilon, H(\mbz))) =  \Hhat^{-1}(H(\mbz))
\end{align*}
\end{proof}

\subsection{Comparion against other identification results}\label{appsec:identification}
\citet{imbens2009identification} consider non-separable outcome processes, i.e.\ $\mby = f(\mbt, \mbz)$ and construct control functions by assuming that 1) treatment is a strictly monotonic function of the confounder, 3) the confounder is continuous with a strictly monotonic \acrshort{cdf}, and 2) positivity holds for $\mbt$ with respect to $\mbz$.
These assumptions also lead to identification with general control functions due to the following:
a) the positivity assumption is equivalent to the strong~\gls{iv} assumption and b) like additivity, the strict monotonicity assumption reflected in the reconstruction map $d(\hz, \mbeps)$ as a function of $\hz$ helps guarantee joint independence; see~\cref{appsubsec:monotonicity}.

\Gls{2sls} requires the outcome process to be additive, $\mby = f(\mbt) + \mbz$.
Further, \gls{2sls} needs a ``completeness'' property: the causal effect function and \gls{iv} are correlated~\citep{guo2016control}.
While joint independence may not be guaranteed by the completeness condition, it can be guaranteed in certain settings that violate completeness.
An example is multiplicative treatment $\mbt = \mbz*\mbeps$ with $\mbz\sim \cN(0,1)$ and a linear outcome; \gls{2sls} fails because $\E[\mbt\mbeps]=0$.
When joint independence can be guaranteed and the \gls{iv} is strong, identification with general control functions does not require structural restrictions like additivity of the outcome process that both~\gls{2sls} and \gls{cfn} rely on.

\subsection{Estimation error bounds}\label{appsec:error-bound}
We give an example of how violations in reconstruction and independence affect errors in effects.
\setcounter{thm}{1}
\subsubsection{\gls{gcfn}'s estimation error in additive treatment process}\label{appsec:add-t-error-bound}
\addprocthm{}
\begin{proof}
  
Recall the true data distribution is $F(\mbt, \mbz, \mbeps)$ such that $\mbz \indep \mbeps$ and the implied joint $q(\hz, \mbt, \mbz, \mbeps) = q(\hz\g \mbt, \mbeps)F(\mbt, \mbz, \mbeps)$.
  For any $L$-Lipschitz function $\ell(\epsilon)$:
  \begin{align}
    \begin{split}
        |\E_{q(\mbeps, \hz)}\ell(\mbeps) \hz| &= |\E_{q(\hz)}\left(\hz\E_{q(\mbeps\g \hz)}\ell(\mbeps)\right)  - \left(\E_{q(\hz)F(\mbeps)}\hz\ell(\mbeps)\right) | \quad \{\E_{q(\hz)}\hz = 0\}
        \\
          &= \left|\E_{q(\hz)}\left(\hz\left(\E_{q(\mbeps\g \hz)}\ell(\mbeps)  - \E_{F(\mbeps)}\ell(\mbeps)\right)\right)\right| 
        \\
          &\leq \E_{q(\hz)}|\hz| \left|\E_{q(\mbeps\g \hz)}\ell(\mbeps)  - \E_{F(\mbeps)}\ell(\mbeps)\right| 
        \\
        & \leq L \E_{q(\hz)}|\hz|\wass{q(\mbeps\g\hz)}{F(\mbeps)}
        \\
        & \leq \gamma L \E_{q(\hz)}|\hz|.
    \end{split}
  \end{align}
  
  Using the definition of the additive treatment process and the reconstruction error bound,
  $\E_{q(\mbz, \hz, \mbeps)}(\mbz+g(\mbeps) - \hz- g^\prime(\mbeps))^2 = \E_{q(\mbt, \hz, \mbeps)}(\mbt - \hz - g^\prime(\mbeps))^2 \leq  \delta$.
  Now, we can bound error in $\hz$ approximating $\mbz$
  \begin{align*}
   \delta &\geq \E_{q(\mbz, \hz, \mbeps)}(\mbz - \hz  + g(\mbeps) - g^\prime(\mbeps))^2
   \\
        &  = \E_{q(\mbz, \hz)}(\mbz - \hz )^2 +  \E_{F(\mbeps)}( g(\mbeps) -g^\prime(\mbeps))^2 +  2\E_{q(\mbz, \hz, \mbeps)}(\mbz-\hz)( g(\mbeps) - g^\prime(\mbeps))
   \\
        &  \geq \E_{q(\mbz, \hz)}(\mbz - \hz )^2 +  2\E_{q(\mbz, \hz, \mbeps)}(\mbz-\hz)( g(\mbeps) - g^\prime(\mbeps))
   \\
        &  = \E_{q(\mbz, \hz)}(\mbz - \hz )^2 +  2\E_{F(\mbz)F(\mbeps)}\mbz( g(\mbeps) - g^\prime(\mbeps)) -  2\E_{q( \hz, \mbeps)}\hz( g(\mbeps) - g^\prime(\mbeps)) \quad\quad \{\mbz\indep \mbeps\}
   \\
        &  = \E_{q(\mbz, \hz)}(\mbz - \hz )^2 + 0 -  2\E_{q(\hz, \mbeps)}\hz( g(\mbeps) - g^\prime(\mbeps)) \quad\quad \{\E_{F(\mbz)}\mbz=0\}
   \\
        &  \geq \E_{q(\mbz, \hz)}(\mbz - \hz )^2  -  4\gamma L_g \E_{q(\hz)}|\hz| \quad \{g(\mbeps) - g^\prime(\mbeps) \text{ is $2L_g$-Lipschitz}\}
  \end{align*} 
  Thus,
  $\E_{q(\mbz, \hz)}(\mbz - \hz )^2 \leq \delta + 4\gamma L_g \E_{q(\hz)}|\hz|.$
  We bound the absolute error in causal effect due to using $\hz$ instead of $\mbz$
  \begin{align}
    \begin{split}
      \E_{\mbt}|\tauhat(\mbt) - \tau(\mbt)|  & = \E_{\mbt} |\E_{q(\hz)}f(\mbt,\hz) - \E_{F(\mbz)} f(\mbt, \mbz)|
      \\
                & =\E_{\mbt} |\E_{q(\hz, \mbz)}\left(f(\mbt,\hz) - f(\mbt, \mbz)\right)| 
      \\
                  & \leq \E_{\mbt} \E_{q(\hz, \mbz)} \left|f(\mbt,\hz) - f(\mbt, \mbz)\right|
      \\
                & \leq \E_{\mbt} L \E_{q(\hz, \mbz)}|\hz - \mbz|
      \\
                & \leq L \E_{\mbt} \sqrt{\E_{q(\hz, \mbz)}\left(\hz - \mbz\right)^2} \quad (\text{Cauchy-Schwarz})
      \\
                & \leq L \sqrt{\delta + 4\gamma L_g \E_{q(\hz)}|\hz|}
    \end{split}
  \end{align}
  When sample size goes to $\infty$, we can guarantee that reconstruction becomes perfect, meaning that $\delta\rightarrow 0 $, and that $\hz \indep \mbeps$ holds, meaning that $\gamma \rightarrow 0$.
  Then, this error bound on effects becomes $0$.

\end{proof}

\subsubsection{Bounding effect estimation error}\label{appsec:gen-err-bound}
 Here, we show that if positivity holds for $\mbt$ w.r.t.\ $\mbz$, and $\mbt$ w.r.t.\ $\hz$, the residual confounding given $\hz$, i.e.\ $\MI(\mbz;\mbt\g \hz)$, controls the expected absolute error in effects if $q(\mbz \g \hz)$ is sufficiently concentrated.
\newcommand\boundthm{
Let $F(\mby, \mbt,\mbz,\mbepsilon)$ be the true data distribution.
Let $q(\mby, \mbt, \hz) = \int F(\mby, \mbt, \mbz=z, \epsilon) q(\hz \g \mbt, \epsilon) d z d \epsilon$.
With $\tau(t^*)$ and $\tauhat(t^*)$ as the true and estimated causal effect of $do(\mbt=t^*)$ respectively, let $\omega(t^*) = |\tauhat(t^*) - \tau(t^*)|$ be the error.
We assume the following.
 \begin{enumerate}[topsep=0pt,leftmargin=*,itemsep=-1mm,partopsep=0ex,parsep=1ex]
   \item Assume that $\mbt$ satisfies positivity with respect to $\mbz$, and $\mbt$ satisfies positivity with respect to $\hz$.
   \item Let $\E[\mby \g \mbt=t, \mbz=z] $ where $\E$ is w.r.t.\ $F$, be an $L_t$-Lipschitz function of $z$, for any $t$.
   \item Let $ L := \sup_t L_t $.
   Let $W := \sup_{t, \zhat} \nicefrac{F(\mbt=t)}{q(\mbt=t\g \hz=\zhat)}$.
   \item 
   Assume $q(\mbz \g \hz)$ satisfies the transportation inequality $T_1(\sigma^2/2)$~\citep{niles2019estimation}.
 \end{enumerate}
 Then, the expected absolute error in effects is bounded as:
  $\quad \E_{F(\mbt)} \omega(\mbt)  \leq \sigma L\sqrt{W \mbI(\mbz;\mbt \g \hz)}.$
}
\begin{thm}\label{thm:bound-thm}
\boundthm{}
\end{thm}
\begin{proof} (of \cref{thm:bound-thm})
  Positivity of $\mbt$ w.r.t.\ $\mbz$ implies the conditional expectation $\E[\mby\g \mbz=z,\mbt=t^*]$ exists for all $z\in \supp(F(\mbz)), t^*\in\supp(\mbt)$.
  Positivity of $\mbt$ w.r.t.\ $\hz$ implies the conditional expectation $\E[\mby\g \hz=\zhat,\mbt=t^*]$ exists for all $\zhat\in \supp(F(\hz)), t^*\in\supp(\mbt)$.
    We begin by expanding the expectation $\E[\mby\g \hz=\zhat,\mbt=t^*]$ as an integral over the conditional $F(\mby\g\mbz,\mbt, \hz) q(\mbz \g \mbt, \hz)$.
  \begin{align*}
    \begin{split}
      \E[y\g \hz=\zhat,\mbt=t^*] & = \int \E\left[\mby\g \mbz=z,\mbt=t^*,\hz=\zhat\right] q(\mbz =z\g \mbt=t^*,\hz=\zhat) dz
      \\
      & = \int \E\left[\mby\g \mbz=z,\mbt=t^*\right] q(\mbz =z\g \mbt=t^*,\hz=\zhat) d z \quad \{\text{by } \mby \indep \hz\g \mbt, \mbz \},
    \end{split}
  \end{align*}
  where the inner expectation is with respect to the conditional distribution $F(\mby\g \mbt, \mbz)$.
  Now, we prove the bound on $\omega(t^*)$ by expanding the true and estimated effects as expectations over $\mbz$:
  \begin{align*}
    \omega(t^*) &= |\tau(t^*) - \tauhat(t^*)|
    \\
    & = \left|\int\left[F(\mbz = z) - \E_{q(\hz)}q(\mbz =z\g \mbt=t^*,\hz)  \right]\E\left[\mby\g \mbz=z,\mbt=t^*\right] d z\,\right|
    \\
    & =  L_{t^*}\left|\int\left[F(\mbz = z) - \E_{q(\hz)}q(\mbz =z\g \mbt=t^*,\hz)\right] \frac{\E\left[\mby\g \mbz=z,\mbt=t^*\right]}{L_t^*}d z\,\right|
    \\
    & =  L_{t^*}\left|\int\left[\E_{q(\hz)}\left(q(\mbz = z\mid \hz)  - q(\mbz =z\g \mbt=t^*,\hz)\right)\right] \frac{\E\left[\mby\g \mbz=z,\mbt=t^*\right]}{L_t^*}d z\,\right|
    \\
    & \leq L_{t^*}\E_{q(\hz)}\left|\int\left[\left(q(\mbz = z\mid \hz) - q(\mbz =z\g \mbt=t^*,\hz)\right)\right] \frac{\E\left[\mby\g \mbz=z,\mbt=t^*\right]}{L_t^*}d z\,\right|
    \\
    & \leq L_{t^* } \E_{q(\hz)}\wass{q(\mbz\g \mbt=t^*,\hz)}{q(\mbz\g \hz)}
    \\ 
    & \quad\quad\left\{\nicefrac{\E\left[\mby\g \mbz=z,\mbt=t^*\right]}{L_{t^*}}\text{ is 1-Lipschitz}\right\}
    \\
    & \leq L_{t^* } \E_{q(\hz)} \sigma\sqrt{\KL{q(\mbz\g \mbt=t^*,\hz)}{q(\mbz\g \hz)}} \\ 
    & \leq L_{t^* } \sigma\sqrt{ \E_{q(\hz)} \KL{q(\mbz\g \mbt=t^*,\hz)}{q(\mbz\g \hz)}}  \quad \{\text{by Cauchy Schwarz}\},
  \end{align*}
  where the $\mathcal{W}_1$ term was bounded by $\kld$ by the assumption that $q(\mbz \g \hz)$ satisfies the transportation inequality $T_1(\sigma^2/2)$~\citep{niles2019estimation}.
  Using $L = \sup_{t} L _t$ and $W = \nicefrac{F(\mbt=t)}{q(\mbt=t\g \hz=\zhat)} $, we can bound the average absolute error
  \begin{align*}
      \E_{F(\mbt)}\omega(\mbt) &\leq  \sigma \E_{F(\mbt)} L_{\mbt } \sqrt{\E_{q(\hz)} \KL{q(\mbz\g \mbt,\hz)}{q(\mbz\g \hz)}}
      \\
            & \leq \sigma L \sqrt{\E_{q(\hz)} \E_{F(\mbt)}  \KL{q(\mbz\g \mbt,\hz)}{q(\mbz\g \hz)} } \quad \{\text{by Cauchy Schwarz}\}
            \\
            & = \sigma L \sqrt{\E_{q(\hz)}  \E_{q(\mbt\g \hz)}  \frac{F(\mbt)}{q(\mbt\mid \hz)}  \KL{q(\mbz\g \mbt,\hz)}{q(\mbz\g \hz)}}
            \\
            & \leq \sigma L\sqrt{W}\sqrt{\E_{q(\hz)}  \E_{q(\mbt\g \hz)} \KL{q(\mbz\g \mbt,\hz)}{q(\mbz\g \hz)}}
            \\
            & = \sigma L\sqrt{W}\sqrt{\mbI(\mbt ;\mbz \g \hz)}
  \end{align*}
\end{proof}

\subsection{Estimation with the Two-stage least-squares method}\label{app:subsec-2sls}
We first describe the general version of \glsreset{2sls}\gls{2sls}.
Let the outcome, treatment and \gls{iv} be $\mby,\mbt',\mbepsilon$ respectively and the true data distribution be $p(\mbt',\mby,\mbepsilon)$.
\begin{enumerate}[leftmargin=*]
    \item In the first-stage, \gls{2sls} learns the distribution $q(\mbt\g \mbepsilon)$.
    Given some class of distributions $Q$, the first-stage can be framed as a maximum-likelihood problem:
    \[q = \argmax_{q'\in Q}\E_{p(\mbt',\mbepsilon)}\log q'(\mbt'\g \mbepsilon)\]
    In our setup, $\mbt$ is the \textit{synthetic treatment} sampled from the conditional distribution $q$ estimated in the first stage.
    \item In the second-stage, \gls{2sls} learns the conditional distribution of the outcome $\mby$ given the \textit{synthetic treatment} $\mbt$ sampled from the conditional $q(\mbt\g \mbepsilon)$ from the first stage.
    Given some class of distributions $G$, \gls{2sls}'s second-stage can be framed as a maximum-likelihood problem:
    \[g = \argmax_{g'\in G} \E_{p(\mby,\mbepsilon)}\E_{q(\mbt\g \mbepsilon)} \log g'(\mby \g \mbt).\]
    The causal effect estimate is then computed as:
    $ f^*(t) = \E_{g(\mby\g \mbt=t)}[\mby].$
\end{enumerate}

Typically in settings with continuous $\mby,\mbt$, both stages of \gls{2sls} are framed and implemented as least-squares regressions instead of maximum-likelihood problems. See \citet{kelejian1971two} for an overview of classical vs. Bayesian \acrlongpl{2sls}.

In this section, we derive an alternate expression for \gls{2sls}'s causal effect estimate $f^*(t)$.
Recall that $\mbt$ is the \textit{synthetic treatment} sampled from the conditional distribution $q$ estimated in the first stage.
We assume that both stages of \gls{2sls} are perfectly solved. Note that $\mbt$ is independently sampled conditioned on $\mbepsilon$. This imposes the following conditional independencies:
\[\mby\indep \mbt\g \mbepsilon,\mbt'\quad and \quad \mbt'\indep \mbt\g \mbepsilon.\]
We marginalize out $\mbt',\mbepsilon$ from the joint $q(\mby,\mbt,\mbt',\mbepsilon)$ to get the dependence of $\mby$ on $\mbt$:
\begin{align}
\begin{split}
    f^*(t)  &= \E[\mby=y\g \mbt=t] \\
    & = \int_{t',\epsilon}yq(\mby=y, \mbt'=t',\mbepsilon=\epsilon\g \mbt=t) d \epsilon dy dt' \\
    & = \int_{t',\epsilon}yp(\mby=y\g \mbt'=t',\mbepsilon=\epsilon, \mbt=t)q(\mbepsilon=\epsilon\g \mbt=t)p(\mbt=t'\g \mbepsilon=\epsilon, \mbt=t) d \epsilon
    dy dt' \\
    & = \int_{t',\epsilon}yp(\mby=y\g \mbt'=t',\mbepsilon=\epsilon)q(\mbepsilon=\epsilon\g \mbt=t)p(\mbt'=t'\g \mbepsilon=\epsilon) d \epsilon dy dt'  \quad  \\
    & \quad \{\text{by } \mbt'\indep \mbt\g \mbepsilon,\, \mbt\indep \mby\g \mbt, \mbepsilon
                \},
\end{split}
\end{align}
which yields
\begin{align}
\begin{split}
f^*(t) & = \int yp(\mby=y\g \mbt'=t',\mbepsilon=\epsilon)q(\mbepsilon=\epsilon\g \mbt=t)p(\mbt'=t'\g \mbepsilon=\epsilon) d \epsilon dy dt' 
\\
& = \E_{q(\mbepsilon\g \mbt=t)}\E_{p(\mbt'\g \mbepsilon)}\E[\mby\g \mbt',\mbepsilon].
\end{split}
\end{align}
This shows that the effect estimated by \gls{2sls} can be rewritten as
\[f^*(t) = \E[\mby\g \mbt=t] =  \E_{q(\mbepsilon\g \mbt=t)}\E_{p(\mbt'\g \mbepsilon)}\E[\mby\g \mbt',\mbepsilon] \]

With this, we show that \gls{2sls}'s estimation is biased when the outcome process might have multiplicative interactions between treatment and confounders.
Consider this data generation:
\[\mbepsilon,\mbz\sim \cN(0,1),\,\,\mbt = \mbepsilon + \mbz,\,\,\mby = \mbt + \mbt^2\mbz.\]
Let $p(\mbt\g \mbepsilon)$ be the learned conditional treatment distribution from a perfectly solved first-stage. We use the reverse conditional $p(\mbepsilon \g \mbt)$.
\gls{2sls}'s causal effect estimate can be rewritten as $f(t) =  \E_{q(\mbepsilon\g \mbt=t)}\E_{p(\mbt'\g \mbepsilon)}\E[\mby\g \mbt',\mbepsilon]$.
The true causal effect is $f(t)=\E_{p(\mbz)}[\mbt + \mbt^2\mbz\g \DO(\mbt=t)]=t$. Note that $\E[\mbepsilon\g \mbt=t] = \E_{\mbz\sim \cN(0,1)}[t - \mbz]= t$. The \gls{2sls}-estimate is $3t\not = t = f(t)$:
\begin{align*}
  f^*(t) &= \E_{q(\mbepsilon\g \mbt=t)}\E_{p(\mbt'\g \mbepsilon)}\E[\mby\g \mbt',\mbepsilon] 
  \\
  & = \E_{q(\mbepsilon \g \mbt=t)}\E_{p(z)}\E[\mby\g \mbt'=\mbz +\mbepsilon,\mbepsilon] 
  \\
  &  = \E_{q(\mbepsilon \g \mbt=t)}\E_{p(z)}[\mbepsilon + \mbz +  (\mbepsilon + \mbz)^2\mbz] 
   = 3t
\end{align*}
This shows \gls{2sls} needs to assume properties of the true outcome and treatment processes.

\subsection{The DeepIV objective}\label{subsec:deepiv}
DeepIV~\cite{DBLP:conf/icml/HartfordLLT17} extends the \acrlong{2sls} to use neural networks in both stages of treatment and outcome estimation. For simplicity, we ignore the covariates $\mbx$. The first stage of DeepIV estimates the conditional density of treatment given the \gls{iv}.
Assuming the first-stage of DeepIV is solved and we have an estimate $p_\theta(\mbt\g \mbepsilon)$, the outcome stage of DeepIV solves the following to obtain an estimate $f_\phi(\mbt)$ for the true causal effect $f(t)=\E [\mby\g \DO(\mbt=t)]$:
\begin{equation}\label{eq:app-2sls}
\min_\phi \E_{\mby,\mbepsilon}[\mby - \E_{p_\theta(\mbt\g \mbepsilon)} f_{\phi}(\mbt)]^2.
\end{equation}
This optimization \cref{eq:app-2sls} has a subtle issue.
We will show that there exist different functions that solve the optimization problem, thereby resulting in different treatment-effect estimates.
Assume that the first stage was solved with $\mbt\sim  p(\mbt\g \mbepsilon)$.
The trouble lies in the fact that \cref{eq:app-2sls} averages the function $f_\phi(\mbt)$ over the distribution $p(\mbt\g\mbepsilon)$.
If there exists a function $f'\not=0$ such that $\E_{p(\mbt\g \mbepsilon)}f'(\mbt) = 0$, both $f$ and $f + f'$ solve the optimization problem in \Cref{eq:app-2sls}.
As there is no way to separate $f$ from functions like $f + f'$, we face a non-identifiability issue.

We show that multiplicative interactions between $\mbepsilon, \mbz$ in the true treatment process is a sufficient condition for such functions $f'$ to exist.
Consider the following data generation with no confounding:
\[\mbepsilon, \mbz\sim \cN(0,1),\,\, \mbt = \mbz\mbepsilon, \,\, \mby = \mbt^2 .\]
Here the true causal effect is $f(t) = t^2$.
We will show that $\E_{p(\mbt\g\mbepsilon)}f(\mbt) = \E_{p(\mbt\g\mbepsilon)}(f(\mbt) + \mbt)$, meaning that both $f(t)$ and $f(t) + t$ solve the optimization problem \cref{eq:app-2sls}.
Notice that $\E[\mbt \g  \mbepsilon] = 0$ and therefore
\[\E_{p(\mbt\g\mbepsilon)}(f(\mbt) + \mbt) = \E[(\mbt^2 + \mbt)\g  \mbepsilon] = \E[\mbt^2\g\mbepsilon]  + \E[\mbt\g \mbepsilon] = \E[\mbt^2\g\mbepsilon] = \E_{p(\mbt\g \mbepsilon)}[f(\mbt)].\]
For any constant $a$, the function $t^2 + at$ also solves the optimization problem in~\cref{eq:app-2sls}. This means that multiple solutions to the DeepIV objective exist that are not the true causal effect.

One potential reason that DeepIV may not run into this non-identifiability issue is that an upper bound of the original proposed objective is solved instead.
To compute gradients for the original optimization, two independent expectations are needed, which is not sample-efficient; this is called the double-sample problem.
So, \citep{DBLP:conf/icml/HartfordLLT17} optimize an upper bound (via Jensen's):
\begin{equation}\label{eq:app-upper-bound}
 \E_{F(\mby,\mbepsilon)}[\mby - \E_{p_\theta(\mbt\g \mbepsilon)} f_{\phi}(\mbt)]^2 \leq \E_{F(\mby,\mbepsilon)}\E_{p_\theta(\mbt\g \mbepsilon)}[\mby -  f_{\phi}(\mbt)]^2.
\end{equation}
The RHS above is a log-likelihood problem with a Gaussian likelihood. A general form of this is $\E_{F(\mby,\mbepsilon)}\E_{p_\theta(\mbt\g\mbepsilon)}\log p_\phi(\mby\g \mbt)$; where $p_\phi$ is supposed to model the distribution of the outcome under $\text{do}(\mbt)$. Finally, as DeepIV is based on \gls{2sls}, DeepIV assumes an additive outcome process to avoid the issues in the previous section.

\paragraph{DeepIV under multiplicative treatment processes}
We show here that the upper bound that DeepIV minimizes can also produce biased effect estimates when the true treatment process is multiplicative.
The upper bound that DeepIV optimizes is:
\[ \arg \min_{f^*} \E_{F(\mby, \mbepsilon)}\E_{p(\mbt\g\mbepsilon)}[\mby - f^*(\mbt)]^2 = \arg \min_{f^*} \E_{F(\mbepsilon)}\E_{p(\mbt\g \mbepsilon)} \E_{F(\mby\g\mbepsilon)} [\mby - f^*(\mbt)]^2 \]
Note that we use $F(\mby\g \mbepsilon)$ and not $F(\mby\g \mbt,\mbepsilon)$ because here $\mbt$ refers to the synthetic treatment sampled from the conditional distribution $p(\mbt\g \mbepsilon)$ learned in the first stage of DeepIV, which means $\mby \indep \mbt\g \mbeps$.
We do a bias-variance decomposition of the expectation and refer to terms that do not depend on $h$ as constants $C$ with respect to the optimization.
\begin{align}
  \begin{split}
    \E_{F(\mbepsilon)}\E_{p(\mbt\g \mbepsilon)} \E_{F(\mby\g\mbepsilon)} [\mby - f^*(\mbt)]^2  & = \E_{F(\mbepsilon)}\E_{p(\mbt\g \mbepsilon)} \E_{F(\mby\g\mbepsilon)} [\E[\mby\g \mbepsilon] - f^*(\mbt)]^2  + \E_{F(\mbepsilon)}[\sigma^2(\mby\g \mbepsilon)]
    \\
		 & = \E_{p(\mbt)}\E_{p(\mbepsilon\g \mbt)} \E_{F(\mby\g\mbepsilon)} [\E[\mby\g \mbepsilon] - f^*(\mbt)]^2  + C
   \end{split}
\end{align}
Now consider the generation process $\mbepsilon,\mbz \sim \mathcal{N}(0,1)$ with the true treatment and outcome generated as $\mbt = \mbepsilon \mbz$ and $\mby = \mbt + \mbz$. Note that $E[\mby \g \mbepsilon = a] = E_{\mbz}[\mbz + a\mbz] =  0$. Therefore the optimization reduces to the following:
\[ \arg\min_{f^*} \E_{p(\mbt)}\E_{p(\mbepsilon\g \mbt)} \E_{F(\mby\g\mbepsilon)} [0 - f^*(\mbt)]^2  = 0 \not = f(t)= t\]
Thus, DeepIV's relaxed optimization problem also needs assumptions on the true treatment process.

\subsection{Information preserving maps  and additional utility constraints}\label{subsec:utility}

A bijective map is one that maps each element in its domain to a unique element in its range. No information can be lost in this process, resulting in bijective transformations being called information-preserving maps.
Information-preserving maps preserve computations that only involve conditioning and expectations; meaning that the causal effect estimate $\E_{\hz}\E [\mby\g \mbt,\hz]$ is preserved.
Therefore we can impose additional distributional utility constraints satisfied by bijective transformations of the general control function $\hz$, without losing the properties of ignorability.

Coupled with flexible over-parametrized modelling, information-preserving maps give us the ability to enforce utility constraints on the latent space of $\hz$.
If there is an outcome-model that works well with data drawn from a normal distribution, one can add an additional term to \gls{vde}'s objective that is the KL divergence between the distribution of $\hz$ and a normal distribution.
If we wanted information about continuity in $\mbt$ to be preserved in $\hz$, we could enforce linear interpolation.
Similarly, we could force an constructed $\hz$ to have a monotonic relation with $\mbt$.
One could enforce multiple constraints from a combination of distances, divergences, ordering and modality constraints.
When used correctly, these constraints trade optimization complexity between outcome-stage and \gls{vde}.

%% file: sections/app-exps.tex
\section{Experimental Details}\label{appsec:exps}
In this section, we expand on the details of experiments presented in~\cref{sec:exps}.
In all experiments, the hidden layers in both encoder and decoder networks have 100 units and use ReLU activations.
The outcome model is also a 2-hidden-layer neural network with ReLU activations unless specified otherwise.
For the simulated data, the hidden layers in the outcome model have 50 hidden units.
We optimize \gls{vde} and outcome-stage for $100$ epochs with Adam; starting with a learning rate of $10^{-2}$ and halving it every 10 epochs if the training error goes up.

\subsection{Selecting $\lambda$}\label{appsec:val}
We discuss here why good $\lambda$ (equivalently $\kappa$) can be selected based on the resulting expected outcome likelihood, i.e.\ the outcome modelling objective, on a heldout validation set.

As \gls{vde}'s control function is constructed as a function of $(\mbt, \mbeps)$, i.e. $\hz = e(\mbt, \mbeps)$, it holds that $\mby \indep \hz \g \mbeps, \mbt$.
So, predicting $\mby$ from $(\hz, \mbt)$, as in \gls{gcfn}, cannot be better than predicting $\mby$ from $(\mbt, \mbeps)$:
\[ \Hb(\mby \g \mbt, \hz) \geq \Hb(\mby \g \mbt, \mbeps, \hz)  = \Hb(\mby \g \mbt, \mbeps).\]
The slack in the inequality is $\Hb(\mby \g \mbt, \hz) - \Hb(\mby \g \mbt, \mbeps, \hz) =  \MI(\mby, \mbeps \g \mbt, \hz)$
and equality holds when $\mby \indep \mbeps \g \mbt, \hz$.
This independence holds in general only if both $\mbz \indep \mbeps \g \hz$ and perfect reconstruction hold; see~\cref{eq:cond-indep-y-e}.
Thus, in general, the expected outcome likelihood achieves maximum only when both perfect reconstruction and conditional independence are satisfied.

In practice, instead of the unconstrained \gls{vde}, we optimize the lower-bound objective in~\cref{eq:lowerbound} on a finite dataset.
Due to local minima or finite-sample error, this lower-bound optimized with a $\kappa$ that is too large may give a $\hz$ that retains little information about $\mbz$ so as keep the $\kld$ small.
Similarly, when $\kappa$ is too small, $\hz$ may memorize $\mbt$ to keep the reconstruction error small without paying much in the $\kappa \times \kld$ term.
In either case, the resulting $\hz$ fails to satisfy one of either perfect reconstruction or conditional independence, meaning that $\mby \nindep \mbeps \g \mbt ,\hz$ in general.
Then, as discussed above, the outcome model cannot achieve the maximum possible expected outcome likelihood.
This insight suggests the following procedure to select good $\kappa$ based on validation outcome likelihood:
\footnote{At first glance, one failure case seems to be when $\hz$ memorizes $\mbeps$ only, leading to $\mby \indep\mbeps \g \mbt, \hz$. However, such a $\hz$ does not help reconstruct $\mbt$ along with $\mbeps$ while resulting in a large $\KL{q(\hz \g \mbt, \mbeps)}{q(\hz)}$.
This leads to a very sub-optimal objective value in \gls{vde}.
As we maximize to solve \gls{vde}, such failure cases do not occur.}:
\begin{enumerate}
  \item Solve \gls{vde} for a collection of $\kappa$ and obtain the control function $\hz_{\kappa}$ for each.
  \item Regress $\mby$ on $\mbt, \hz_{\kappa}$ and evaluate expected outcome likelihood on a heldout validation set.
  (This heldout set should be different from the one used to tune all other hyperparameters)
  \item Select the $\kappa$ that led to the largest validation outcome likelihood; use the corresponding  $\hz_{\kappa}$ in \gls{gcfn}'s second stage to estimate effects (retrain or use the model from step 2).
\end{enumerate}

\subsubsection{Conditional Independence of outcome and instrument given $\hz, \mbt$}\label{eq:cond-indep-y-e}
By definition, the potential outcome $\mby_{\mbt}$ depends only on $\mbz$ and for any observed $(\mbt, \mby)$, and by consistency, $\mby = \mby_{\mbt}$. Therefore
$\mbz \indep \mbeps \g \hz, \mbt  \implies \mby_{\mbt} \indep \mbeps \g \hz, \mbt  \Longleftrightarrow \mby \indep \mbeps \g \hz, \mbt .$
Under the joint $q(\hz, \mbt, \mbeps, \mbz) = q(\hz \g \mbt, \mbeps) F(\mbt, \mbeps, \mbz)$, 
it follows that $\mbz \indep \mbeps \g \hz, \mbt $ when the reconstruction property and the conditional independence $\mbz \indep \mbeps \g \hz$ hold:
\begin{align}\label{eq:cond-indep-z-e}
  \begin{split}
    q(\mbz, \mbeps  \g\hz, \mbt) &= q(\mbz \g \mbeps, \hz, \mbt)  q(\mbeps \g \hz, \mbt)
    \\
        &  =  q(\mbz \g \mbeps, \hz)  q(\mbeps \g \hz, \mbt) \quad \{\text{by reconstruction } \mbt = d(\hz, \mbeps) \}
    \\ 
        &  =  q(\mbz \g \hz)  q(\mbeps \g \hz, \mbt) \quad \{\text{by joint independence }  \mbz \indep \mbeps \g \hz\}
    \\
        &  =  q(\mbz \g \hz, \mbt)  q(\mbeps \g \hz, \mbt) \quad \{\text{by joint independence and reconstruction } \mbz \indep \mbt \g \hz\},
  \end{split}
\end{align}
where $\mbz \indep \mbt \g \hz$ is shown in the proof of~\cref{thm:intro}.
If $\mby_{\mbt}$ is an invertible function of $\mbz$,  $\mby \indep \mbeps \g \hz, \mbt \implies \mbz \indep \mbeps \g \hz, \mbt$. 
Thus, in general, $\mbz \indep \mbeps \g \hz, \mbt$ is a necessary condition for $\mby \indep \mbeps \g \hz, \mbt$.

\subsection{Simulations with Specific Decoder Structure}\label{appsubsec:sims}
We used the python package \textit{statsmodels} for \gls{2sls} and our own implementation of \gls{cfn}.
We used the \href{https://deepiv.readthedocs.io/en/latest/readme.html}{DeepIV package} developed by \citet{DBLP:conf/icml/HartfordLLT17}.

\paragraph{Multiplicative treatment +  Additive outcome.}
We use the \gls{2sls} function from statsmodels~\citep{seabold2010statsmodels} which uses a linear model $\mbt  = \beta\mbepsilon + \mbeta_{\mbt}$ that will correctly predict that $\E[\mbt\g \mbepsilon] = 0$.
We optimized the treatment and the response models in DeepIV~\citep{DBLP:conf/icml/HartfordLLT17} for a 100 epochs each.

\subsection{\gls{gcfn} on high-dimensional covariates}\label{appsec:mnist-exp}
Here, we give further details about~\cref{sec:mnist-exp}.
We give~\citet{DBLP:conf/icml/HartfordLLT17}'s simulation with our notation:
\begin{align*}
  \mbz, \mbeps \sim \cN(0,1) ,\quad  \mbt = 25 +  (\mbeps+3)\psi_s + \nu, \quad \mby = \cN\big(100 + (10 + \mbt)\ell(\mbx) \psi_s - 2\mbt + 0.5 \mbz, 0.75\big),
\end{align*}
where $\psi_s$ is a non-linear function of time $s$, and $\ell(\mbx)$ is the label of the MNIST image.
We optimized both \gls{vde} and outcome stage with Adam with batch size $500$ for $200$ epochs beginning at $10^{-2}$ and halving the learning rate when the average loss over $5$ epochs increases.
We use the outcome model architecture from DeepIV~\cite{DBLP:conf/icml/HartfordLLT17} where convolutional layers construct a representation which is concatenated with $\mbt$ and $s$, before being fed to the fully-connected layers.
\Gls{gcfn}'s outcome model differs only in that the fully-connected layers take as input the control function $\hz$, time $s$ and treatment $\mbt$.
The best outcome model was chosen based on validation outcome MSE.

\subsection{\gls{gcfn} on high-dimensional \gls{iv}}\label{appsec:mnist-iv-exp}
\begin{wrapfigure}[11]{r}{0.42\textwidth}
  \vspace{-15pt}
  \centering
   \includegraphics[width=0.42\textwidth]{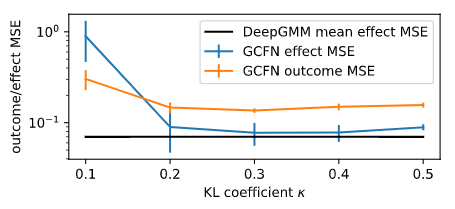}
  \captionof{figure}{\gls{gcfn} performs on par with DeepGMM on high-dimensional IV experiment specified in DeepGMM~\citep{bennett2019deep}.}
  \label{fig:mnist-iv-exp}
\end{wrapfigure}
Here, we give further details about~\cref{sec:mnist-iv-exp}.
The encoder and additive decoder in \gls{vde} are $2$-layer networks like in the~\cref{sec:decoder-structure}.
In this experiment we use a $3$ layer outcome model with $50$ units in each layer.
We used $10,000$ samples as in DeepGMM and optimized both \gls{vde} and outcome stage with Adam with a batch size $1000$ for $100$ epochs beginning at a learning rate of $10^{-2}$ and halving it when the average loss over $5$ epochs increases.
We plot outcome and effect MSE for \gls{gcfn} for $5$ different $\kappa\in\{0.1, 0.2, 0.3, 0.4, 0.5\}$ in~\cref{fig:mnist-iv-exp}.
Note that low outcome MSE corresponds to low effect MSE.
The plot shows mean and standard deviation of effect MSE of the causal effect for $5$ different $\alpha$'s and 10 random seeds.
\gls{gcfn} performs on par or better than all methods given in DeepGMM~\citep{bennett2019deep}.

\subsection{Additional experiments}
The following experiment is done with a structurally unrestricted decoder even though the true treatment process is additive.
We compare against \gls{cfn} to demonstrate that \gls{gcfn} does not require structural restrictions on the outcome process.
Let $\mathcal{N}$ be the normal distribution and $\alpha$ be a parameter to control the confounding strength. We generate 
\begin{align}\label{eq:app-cfnexp-gen}
 \mbz, \mbepsilon \sim \cN(0,1),\quad \mbt = (\mbz + \mbepsilon)/\sqrt{2}, \quad \mby \sim \cN(\mbt^2  + \alpha \mbz^2, 0.1).   
\end{align}
The larger the absolute values of $\alpha$, the more the confounding. 
In economics terminology, the treatment noise and the outcome noise are $\mbeta_t = \mbz$ and $ \mbeta_y= \alpha \mbz^2 + noise$ respectively.
The generation process in \cref{eq:app-cfnexp-gen} violates assumption A4 in \citet{guo2016control} for \gls{cfn}: $\E[\mbeta_y|\mbeta_t] \propto \mbeta_t$.  
\gls{gcfn} does not require this assumption.
We use $5000$ samples and a batch size of $500$.
We discretize the treatment to have $50$ categories.
Of the 50, $48$ categories correspond to equally sized bins in $[-3.5, 3.5]$, with the remaining $2$ correspond to values less than $-3.5$ and greater than $3.5$ respectively.
We compare against \gls{cfn} with both stages correctly specified as functions of $\mbt$ and $\mbz$.

We find, as expected, that \gls{gcfn} out-performs \gls{cfn}. Over 5 runs, for $\alpha=1$, we obtain an RMSE of $\boldsymbol{0.3 \pm 0.1}$ while the \gls{cfn} only manages to obtain an RMSE of $\boldsymbol{1.5 \pm 0.1}$ despite having the correctly specified model for $\mbt^2$.
For other $\alpha \in \{-2,-1,2\}$, \gls{gcfn} was similarly better.